\RequirePackage{fix-cm}
\documentclass[smallextended]{svjour3}       % onecolumn (second format)
\smartqed  % flush right qed marks, e.g. at end of proof
\usepackage{amsmath,amssymb}
\usepackage{amsfonts}
\usepackage{graphicx}
\usepackage{color}
\usepackage{csquotes}
\usepackage{dsfont}
\usepackage{hyperref}
\usepackage{algpseudocode}
\usepackage[round]{natbib}
\usepackage[linesnumbered,ruled,vlined]{algorithm2e}
\newcommand\numberthis{\addtocounter{equation}{1}\tag{\theequation}}

\newtheorem{subdefinition}{Definition}[definition]

\begin{document}

\title{K-NN active learning under local smoothness assumption%\thanks{Grants or other notes
%about the article that should go on the front page should be
%placed here. General acknowledgments should be placed at the end of the article.}
}
%\subtitle{Do you have a subtitle?\\ If so, write it here}

%\titlerunning{Short form of title}        % if too long for running head

\author{Boris Ndjia Njike         \and
        Xavier Siebert %etc.
}

%\authorrunning{Short form of author list} % if too long for running head

\institute{F. Author \at
	Affiliation: University of Mons, Department of Mathematics and Operational Research \\
	Tel.: +32 65-374677\\
	Fax: +32 65-374677\\
	\email{borisedgar.ndjianjike@umons.ac.be}  \\
	ORCID ID: https://orcid.org/0000-0002-2122-994 %of F. Author  %  if needed
	\and
	S. Author \at
	Affiliation: University of Mons, Department of Mathematics and Operational Research \\
	Tel.: +32 65-374690\\
	Fax: +32 65-374690\\
	\email{xavier.siebert@umons.ac.be} 
}
\date{Received: date / Accepted: date}
% The correct dates will be entered by the editor

\maketitle

\begin{abstract}
There is a large body of work on convergence rates either in passive or in active learning. Here we first
outline some of the main results that have been obtained, more specifically in a nonparametric setting
under assumptions about the smoothness of the regression function (or the boundary between classes) and the margin noise. We discuss the relative merits of these underlying assumptions by putting active learning in perspective with recent work on passive learning. In particular, the regression function is assumed to be smooth with respect to the marginal  probability of the data. This is more general than what was previously used in the literature. Our contribution is an active nearest neighbor learning algorithm, that is able to deal with this setting, and that outperforms its passive counterpart. Our algorithm works for a larger class of probability distributions than those previously used, especially for distributions of probability for which the density function is not necessarily bounded below, and also for discrete distributions.
% \PACS{PACS code1 \and PACS code2 \and more}
% \subclass{MSC code1 \and MSC code2 \and more}
\end{abstract}

\section{Introduction}
\label{intro}
Active learning is a machine learning approach for reducing the data labeling effort. Given an instance space $\mathcal{X}$ or a pool of unlabeled data $\lbrace X_1,\ldots,X_w\rbrace$ provided by a distribution $P_X$, the learner focuses its labeling effort only on the most \enquote{informative} points so that a model built from them can achieve the best possible guarantees \citep{dasgupta2011two}. Such guarantees are particularly interesting when they are significantly better than those obtained in passive learning \citep{hanneke2015minimax}. In the context of this work, we consider binary classification (where the label $Y$ of $X$ takes its value in $\lbrace 0,1\rbrace$) in  a nonparametric setting. Extensions to multiclass classification and adaptive algorithms are discussed at the end of this paper (Section~\ref{sec:conclusion}).

The nonparametric setting has the advantage of providing guarantees with many informations such as  the dependence on the dimensional and distributional parameters by using some hypotheses on the regularity of the decision boundary \citep{castro2008minimax}, on the regression function \citep{minsker2012plug, locatelli2017adaptivity}, and on the geometry of instance space (called strong density assumption) \citep{audibert2007fast, locatelli2017adaptivity, minsker2012plug}. 
One of the initial works on nonparametric active learning \citep{castro2008minimax} assumed that the decision boundary is the graph of a smooth function, that a margin assumption very similar to Tsybakov's noise assumption \citep{mammen1999smooth} holds, and that distribution $P_X$ is uniform. This led to a better guarantee than in passive learning. Instead of the assumption on the decision boundary, other works \citep{minsker2012plug, locatelli2017adaptivity} supposed rather that the regression function is smooth (in some sense). This assumption, along with Tsybakov's noise assumption and the strong density assumption also gave a better guarantee than in passive learning. Moreover, unlike in \citep{castro2008minimax}, they provided algorithms that are adaptive with respect to the margin's noise and to the smoothness parameters.\\
However, recent work \citep{chaudhuri2014rates} pointed out some disadvantages of the preceding smoothness assumption, and extended it in the context of passive learning with $k$-nearest neighbors ($k$-NN)   by using a more general smoothness assumption that is able to sharply characterize the rate of convergence for all probability distributions that satisfy it.\\
In this paper, we  thus extend the work of \citep{chaudhuri2014rates} to the active learning setting, and provide a novel algorithm that outputs a classifier with the same rate of convergence as other recent algorithms with more restrictive hypotheses, as for example \citep{minsker2012plug, locatelli2017adaptivity}.
% XS TODO : vendre mieux le résultat final : quelle est la nouveauté majeure, les perspectives, ...
Section~\ref{sec:definitions} introduces general definitions, Section~\ref{sec:convergenceRates} presents previous work on convergence rates in active and passive non-parametric learning, with a special emphasis on the assumptions related to our work. Section~\ref{sec:algorithm} provides an outline of our algorithm while Section~\ref{sec:guarantee} describes its theoretical motivations and Section~\ref{sec:conclusion} contains the conclusion and some perspectives for future work.  

\section{Preliminaries}
\label{sec:definitions} 
We begin with some general definitions and notations about active learning in binary classification, then recall the concept of $k$-NN classifiers. Finally, the main assumptions that are used in nonparametric active learning are explained. 
\subsection{Active learning setting}
\label{sec:ALsetting} 
%\textcolor{red}{TODO: add a few more references -- en particulier un papier published de Hanneke !!}

Let $(\mathcal{X}, \rho)$ be a metric space. In this paper we set $\mathcal{X}\subset\mathbb{R}^d$ %\footnote{In general $\mathcal{X}$ is a metric space, and the technicals, topologicals property as separability are often required.} 
and refer to it as the instance space, and take $\rho$ as the Euclidean metric. Let $\mathcal{Y}=\lbrace 0,1\rbrace$ the label space. We assume that the pairs $(X, Y)$ are random variables distributed according to an unknown probability $P$ over $\mathcal{X}\times \mathcal{Y}$. Let us denote $P_X$ the marginal distribution of $P$ over $\mathcal{X}$. 

Given $w$ $\in$ $\mathbb{N}$ and an i.i.d. sample $(X_1,Y_1),\ldots, (X_w,Y_w)$ drawn according to  probability $P$, the learning problem consists in minimizing the risk $\mathcal{R}(f)=P(Y\neq f(X))$ over all measurable functions, called classifiers, $f:\mathcal{X}\rightarrow \mathcal{Y}$.

In active learning, the labels are not available from the beginning but we can request iteratively at a certain cost (to a so-called oracle) a given number $n$ of samples, called the budget ($n\leq w$).  In passive learning, all labels are available  and $n=w$.
At any time, we choose to request the label of a point $X$ according to the previous observations. The point $X$ is chosen to be most \enquote{informative}, which amounts to belonging to a region where classification is difficult and requires more labeled data to be collected. Therefore, the goal of active learning is to design a sampling strategy that outputs a classifier $\widehat{f}_{n,w}$ whose excess risk (see below) is as small as possible with high probability over the requested samples, as reviewed in  \citep{dasgupta2011two, hanneke2015minimax, dasgupta2017active}.

Given $x$ in $\mathcal{X}$, let us introduce the regression function $\eta(x)=\mathbb E(Y\vert X=x)=P(Y=1\vert\,X=x)$. It is easy to show \citep{lugosi2002pattern} that the function $f^*(x)=\mathds{1}_{\eta(x)\geq 1/2}$ achieves the minimum risk and  that $\mathcal{R}(f^*)=\mathbb E_X(\min(\eta(X),1-\eta(X)))$. Because $P$ is unknown, the function $f^*$ is unreachable and thus the aim of a learning algorithm is to return a classifier $\widehat{f}_{n,w}$ with minimum excess risk $\mathcal{R}(\widehat{f}_{n,w})-\mathcal{R}(f^*)$ with high probability over the sample $(X_1,Y_1),\ldots, (X_w,Y_w)$.
\subsection{$k$-Nearest Neighbors ($k$-NN) classifier}

Given two integers $k,\;n$ such that $k<n$, and a test point $X$ $\in$ $\mathcal{X}$, the $k$-NN classifier predicts the label of $X$ by giving the majority vote of its $k$ nearest neighbors amongst the sample $X_1,\ldots,X_n$. For $k=1$, the $k$-NN classifier returns the label of the nearest neighbor of $X$ amongst the sample $X_1, \ldots, X_n$. If $k$ is allowed to grow with $n$, the method is called $k_n$-NN. For a complete discussion of nearest neighbors classification, see for example \citep{biau2015lectures,shalev2014understanding,chaudhuri2014rates}.

\subsection{Regularity, noise and strong density assumptions}
\label{sec:notation-def}
%As mentioned in the abstract, we are interested in rate of convergence of a classifier output by an active algorithm under assumption from the work of Chaudhuri and Dasgupta \citep{chaudhuri2014rates}. We are going to introduce some notations and definitions about these assumptions.\\
Let $B(x,r)=\lbrace x'\in \mathcal{X},\; \rho(x,x')< r\rbrace$ and $\bar B(x,r)=\lbrace x'\in \mathcal{X},\; \rho(x,x')\leq r\rbrace $  the open and closed balls (with respect to the Euclidean metric $\rho$), respectively, centered at $x$ $\in$ $\mathcal{X}$ with radius $r>0$. Let $\text{supp}(P_X)=\lbrace x\in \mathcal{X},\,\;\forall r>0,\;P_X(B(x,r))>0\rbrace$  the support of the marginal distribution $P_X$.

\setcounter{definition}{1}

\begin{subdefinition}[H\"{o}lder continuity]~\\
Let $\eta: \mathcal{X}\rightarrow [0,1]$ be the regression function defined as $\eta(x)=P(Y=1 \vert X=x)$. We say that $\eta$ is $(\alpha,L)$-\textbf{H\"{o}lder continuous} $(0<\alpha\leq 1, \text{and}\; L\geq 1)$ if $\forall$ $x,x'\in \mathcal{X}$,   
\begin{equation}
\label{def:Holder}
\vert \eta(x)-\eta(x')\vert\leq L\rho(x,x')^{\alpha}. \tag{H1a} 
\end{equation}

\end{subdefinition} 
The notion of H\"{o}lder continuity ensures that the proximity between  two closest (according to the metric $\rho$) points is reflected in a similar value of the conditional probability $\eta$.\\
This definition remains true for a general metric space, but when $\rho$ is the Euclidean metric, we should always have $0<\alpha\leq 1$, otherwise $\eta$ becomes constant\citep{pugh2002real}.\\
In most of the previous works (for example \citep{audibert2007fast,minsker2012plug,gai2007sharp}), the definition \ref{def:Holder} is used along with the following notion \eqref{def:strongDensity} for technical reasons. 

%\textcolor{red}{briefly explain technical reasons ?}\textcolor{blue}{okay,  plus bas, juste apres la def 1b}
% idee de la preuve. se donner [x,y], subdivisionner en $n$ intervalles egaux, et montrer que \eta-\eta(y)<1/n^{\alpha-1}. ensuite faire tendre n vers l'infini.
 
%This definition can be generalized for $\alpha>1$; in this work, we only deal with the case $\alpha\leq 1$. 

\begin{subdefinition}[Strong density]\label{def:strong}~\\
Let $P$ be the  probability distribution defined over $\mathcal{X}\times \mathcal{Y}$ and $P_X$ the marginal distribution of $P$ over $\mathcal{X}$. We say that $P$ satisfies the \textbf{strong density} assumption if there exists some constants $r_0>0$, $c_0>0$, $p_{min}>0$ such that for all $x$ $\in$ $\text{supp}(P_X)$: 
\begin{equation}
\begin{split}
&\lambda(B(x,r)\cap \text{supp}(P_X))\geq c_0 \lambda(B(x,r)),\;\forall r\leq r_0\\
&\text{and}\; p_X(x)>p_{min},
\end{split}
\tag{H1b}
\label{def:strongDensity}
\end{equation}
where $p_X$ is the density function of the marginal distribution $P_X$ and $\lambda$ is the Lebesgue measure.
\end{subdefinition}
The strong density assumption ensures that, given a realisation $X=x$ according to $P_X$, there exists an infinite number of realisations $X_1=x_1,\ldots, X_m=x_m,\ldots$ in a neighborhood of $x$. \\Sometimes, the notion of strong density is used  to geometrically characterize the set where the classification is difficult \citep{locatelli2017adaptivity}, and then combined with the following definition of \emph{Margin noise}, allows to nicely control the error  of classification beyond a given number of label requests. 
% In the context of machine learning, it amounts to saying that there exists a large number of training points in the vicinity of any test point.

\begin{definition}[Margin noise]\label{Tsy}~\\
The probability distribution $P$ satisfies the \textbf{margin noise} assumption (sometimes called \textbf{Tsybakov's noise} assumption \citep{audibert2007fast}) with parameter $\beta\geq 0$ if for all $0<\epsilon\leq 1$, there is $C=C(\beta)$ $\in$ $[1, +\infty[$ such that

\begin{equation}
\label{def:TsybakovMarginNoise}
P_X(x\in\mathcal{X},\;\vert\eta(x)-1/2\vert\leq\epsilon)<C\epsilon^{\beta}. \tag{H2} 
\end{equation} 
\end{definition}
The margin noise assumption gives a bound on the probability that the label of the points in the neigborhood of a point $x$ differs from the label of $x$ given by the conditional probability $\eta(x)$. It also describes the behavior of the regression function in the vicinity of the decision boundary $\eta(x)=\frac 12$. When $\beta$ goes to infinity, we observe a \enquote{jump} of $\eta$ around the decision boundary, and then we obtain Massart's noise condition \citep{massart2006risk}. Small values of $\beta$ allow for $\eta$ to \enquote{cuddle} $\frac 12$ when we approach the decision boundary.     
% TODO cuddle ?

\begin{definition}[$(\alpha,L)$-smooth]~\\
Let $0<\alpha\leq 1$ and $L>1$. The regression function is \textbf{$(\alpha,L)$-smooth} if for all $x,z$ $\in$ supp$(P_X)$ we have: 
\begin{equation}
\label{def:smooth}
\begin{split}
&\vert \eta(x)-\eta(z)\vert \leq L.P_X(B(x,\rho(x,z)))^{\alpha/d},   
\end{split} 
\tag{H3}
\end{equation}
where $d$ is the dimension of the instance space.
%where $$\eta(B(x,r))=\frac{1}{P_X(B(x,r))}\int_{B(x,r)} \eta(x')dP_X(x').$$
\end{definition} 
Equivalently, \eqref{def:smooth} can be rewritten as: 
\begin{equation*}
%\label{def:smooth-bis}
\begin{split}
&\vert \eta(x)-\eta(z)\vert \leq L.\min\left(P_X(B(x,\rho(x,z)))^{\alpha/d},P_X(B(z,\rho(x,z)))^{\alpha/d}\right).
\end{split} 
\end{equation*}
It is important to note that the $(\alpha,L)$-smooth assumption \eqref{def:smooth} is more general than the H\"{o}lder continuity assumption \eqref{def:Holder}, as stated in Theorem~\ref{theo1} below.

\begin{theorem}\citep{chaudhuri2014rates}\label{theo1}~\\ Suppose that $\mathcal{X}\subset\mathbb{R}^d$, that the regression function $\eta$ is $(\alpha_h,L_h)$-H\"older continuous, and that $P_X$ satisfies \ref{def:strongDensity}. Then there is a constant $L>1$ such that for any $x,z$ $\in$ supp($P_X$), we have: 
$$\vert \eta(x)-\eta(z)\vert \leq L.P_X(B(x,\rho(x,z)))^{\alpha_h/d}.$$
\end{theorem} 
%Let us note that, in general, when $\mathcal{X}\subset \mathbb{R}^d$,  as suggested in \citep{chaudhuri2014rates},  we must %have $B(x,r)\subset\mathcal{X}$ for all $x$ $\in$ $\mathcal{X}$ and $r>0$; or more technically,  $\mathcal{X}$ contains a %constant fraction of every ball centered in it.

%\begin{definition}{Knn classifier}~\\
%Let a set of training points $(x_1,y_1),\ldots,(x_n,y_n)$ with $x_i$ $\in$ $\mathcal{X}$ and $y_i$ $\in$ $\mathcal{Y}$. Given $k<n$, and a test point $x$ $\in$ $\mathcal{X}$, the $k$-NN classifier predicts the label of $x$ by giving the majority vote of its $k^{th}$ nearest neighbors amongst the $x_i$.
%\end{definition}
%For $k=1$, the $k$-NN classifier returns the label of the nearest neighbor of $x$ amongst the $x_i$. Often $k$ $k$ grows with $n$, in this case we say $k_n$-NN classifier and the definition remains the same as $k$-NN. 
\begin{definition}[Doubling-probability]~\\
The marginal distribution $P_X$ is a \textbf{doubling-probability} if there exists a constant $C_{db}>0$ such that for any $x \in \mathcal{X}$, and $r>0$, we have: 
\begin{equation}
\label{def:doubling}
P_X(B(x,r))\leq C_{db} P_X(B(x,r/2)).
\tag{H4}
\end{equation} 

\end{definition}
This notion was initially introduced for geometric purposes in the setting of measure theory \citep{heinonen2012lectures,federer2014geometric}. It helps for constructing a subcover of a metric space by also minimizing the overlap between the elements of the subcover.  Doubling-probability has been used in a machine learning context, particularly $k$-NN classification (or regression), where the constant $C_{db}$ is interpreted as the intrinsic dimension of the region where the data belong \citep{kpotufe2011k}. 
%Indeed, he considered the high-dimensional problem, and provided the results only in terms of the intrinsic dimension. 
This allows to reduce considerably the complexity of the classification problem and to bypass the so-called curse of dimension.  Also, it is also proved \citep{kpotufe2011k} that the notion of doubling-probability generalizes the strong density assumption \ref{def:strongDensity}. It is thus more universal, and does not require a probability density. \\In this paper, doubling-probability is used only for geometrical purposes. It is later relaxed, so that it becomes sufficient to consider only balls $B(x,r)$ with $P_X(B(x,r))$ sufficiently large to satisfy the doubling-probability condition~\eqref{def:doubling}.
\section{Convergence rates in nonparametric active learning}
\label{sec:convergenceRates} 

\subsection{Previous work}
Active learning theory has been mostly studied during the last decades in a parametric setting, see for example \citep{balcan2010true, hanneke2011rates, dasgupta2017active} and references therein.
One of the pioneering works studying the achievable limits in active learning in a nonparametric setting \citep{castro2008minimax} required that the decision boundary is the graph of a H\"{o}lder continuous function with parameter $\alpha$ \eqref{def:Holder}. Using a notion of margin noise (with parameter $\beta$) very similar to \eqref{def:TsybakovMarginNoise}, the following minimax rate was obtained:
\begin{equation}
O\left(n^{-\frac{\beta}{2\beta+\gamma-2}}\right)\label{rate1},
\end{equation}
where $\gamma=\frac{d-1}{\alpha}$ and $d$ is the dimension of instance space $(\mathcal{X}=\mathbb{R}^d)$.\\

Note that this result assumes the knowledge of the smoothness and margin noise parameters, whereas an algorithm that achieves the same rate, but that adapts to these parameters was proposed recently in \citep{locatelli2018adaptive}.  

%\textcolor{red}{Most of the previous works related to nonparametric active learning suppose either that the regression function is smooth, or that the decision boundary is defined by the graph of smooth function.}

%In this paper, we consider the case where the smoothness assumption refers to the regression function both in passive and in active learning.

In passive learning, by assuming that the regression function is H\"{o}lder continuous \eqref{def:Holder}, along with \eqref{def:strongDensity} and \eqref{def:TsybakovMarginNoise},  the  following minimax rate was established  \citep{audibert2007fast}: 
\begin{equation}\label{rate3}
O\left(n^{-\frac{\alpha(\beta+1)}{2\alpha+d}}\right).
\end{equation}

In active learning, using the same assumptions \eqref{def:Holder}, \eqref{def:strongDensity} and \eqref{def:TsybakovMarginNoise}, with the additional condition $\alpha\beta<d$,  the following minimax rate was obtained \citep{locatelli2017adaptivity} 
\begin{equation} \label{rate2}
\tilde{O}\left(n^{-\frac{\alpha(\beta+1)}{2\alpha+d-\alpha\beta}}\right),
\end{equation}  
where $\tilde{O}$ indicates that there may be additional logarithmic factors.
This active learning rate given by \eqref{rate2} thus represents an improvement over the passive learning rate \eqref{rate3} that uses the same hypotheses.

With another assumption on the regression function relating the $L_2$ and $L_{\infty}$ approximation losses of certain piecewise constant or polynomial approximations of $\eta$ in the vicinity of the decision boundary,  the same rate \eqref{rate2} was also obtained \citep{minsker2012plug}.
%A lot of the previous work focused on the H\"{o}lder continuity assumption  (Definition \ref{Holder}) on the regression function, either in active or passive learning.

\subsection{Link with $k$-NN classifiers}
For practicals applications, an interesting question is whether $k$-NN classifiers attain the rate given by  \eqref{rate3}  in passive learning and by \eqref{rate2} in active learning. 

In passive learning, under assumptions \eqref{def:Holder}, \eqref{def:strongDensity} and \eqref{def:TsybakovMarginNoise}, and for suitable $k_n$, it was shown in \citep{chaudhuri2014rates} that $k_n$-NN indeed achieves the rate \eqref{rate3}.

In active learning a pool-based algorithm that outputs a $k$-NN classifier has been proposed in \citep{kontorovich2016active}, but its assumptions differ from ours in terms of smoothness and noise. Similarly, the algorithm proposed in \citep{hannekenonparametric} outputs a 1-NN classifier based on a subsample of a given pool of data, such that the label of each instance of this subsample is determined with high probability by the labels of its neighbors within the pool. The number of neighbors is adaptively chosen for each instance in the subsample, leading to the minimax rate \eqref{rate2} under the same assumptions as in \citep{locatelli2017adaptivity}. 

To obtain more general results on the rate of convergence for $k$-NN classifiers in metric spaces under minimal assumptions, the more general smoothness assumption given by \eqref{def:smooth} was used in \citep{chaudhuri2014rates}.
By using a $k$-NN algorithm, and under assumptions  \eqref{def:TsybakovMarginNoise} and \eqref{def:smooth}, the rate of convergence obtained in \citep{chaudhuri2014rates} is also of the order of \eqref{rate3}. Additionally, using assumption \eqref{def:smooth} instead of \eqref{def:Holder} removes the need for the strong density assumption \eqref{def:strongDensity}, which therefore allows for more probability classes. %: $$O(n^{-\alpha(\beta+1)/(2\alpha+d)}).$$
%It is worth noting that this rate is the equivalent (by applying Theorem~\ref{theo1}) to the rate given by  \eqref{rate3}. % obtained by Audibert and Tsybakov \citep{audibert2007fast}.

%Additionally, this rate avoids the strong density assumption \eqref{def:strongDensity} and therefore allows more classes of probability. In addition, the $\alpha$-smooth assumption is more universal than H\"older continuity assumption. It just holds for any pair of distributions $P_X$ and $\eta$. In H\"{o}lder continuity, strong density assumption implicitly assumes the existence of the density $p_X$, and according to \eqref{def:strongDensity}, it also implies that the support of $P_X$ has finite Lebesgue measure; this is very restrictive and excludes important densities like Gaussian densities as noticed in \citep{doring2017rate}.
%\textcolor{red}{relire le dernier paragraphe, repetitions...}

\subsection{Contributions of the current work}
%\textcolor{red}{XS TODO : still needs some work. Also it is weird to have $i \neq j$ for Assumption $i$ and the corresponding $H_j$}
%\textcolor{blue}{A discuter oralement}
In this work, we will use the assumptions that were used in the context of passive learning in \citep{chaudhuri2014rates}, and show that is is possible to use them in active learning as well. 

For the sake of clarity, let us restate here these assumptions that will be used throughout this paper.
We assume that the assumptions \eqref{def:smooth}, \eqref{def:TsybakovMarginNoise}, \eqref{def:doubling} simultaneously hold respectively with parameters $(\alpha,L)$, $(\beta,C)$, $C_{db}$.

In this paper, we provide an active learning algorithm under assumptions \eqref{def:smooth}, \eqref{def:TsybakovMarginNoise} that were used in passive learning in \citep{chaudhuri2014rates}. We additionally assume that the underlying marginal probability $P_X$ satisfies \eqref{def:doubling} mostly for geometrical convenience. 
Our algorithm has several advantages: 

%\textcolor{red}{XS TODO: when the text is clean we can remove the itemize}
\begin{itemize}
\item[$\bullet$] The assumption \ref{def:smooth}  involves a dependence on the marginal distribution $P_X$, and holds for any pair of distributions $P_X$ and $\eta$, which allows the use of discrete probabilities. However in active learning the H\"{o}lder continuity notion \eqref{def:Holder} is typically used, along with the strong density notion \eqref{def:strongDensity} \citep{castro2008minimax, locatelli2017adaptivity,minsker2012plug}, which implies assuming the existence of the density $p_X$ of the marginal probability $P_X$. By 
using assumption \eqref{def:smooth} instead of \eqref{def:Holder} and thereby avoiding \eqref{def:strongDensity}, our algorithm removes unnecessary restrictions on the distribution that would exclude important densities (e.g., Gaussian) as noticed in \citep{doring2017rate}.
\item[$\bullet$] The rate of convergence of our algorithm is better than those obtained in passive learning under \eqref{def:smooth} and \eqref{def:TsybakovMarginNoise}. 
\item[$\bullet$] According to the assumption \eqref{def:TsybakovMarginNoise}, as we will see, our algorithm also (as in \citep{minsker2012plug}) covers the most interesting case where the regression function is allowed to cross the boundary decision $\lbrace x,\;\eta(x)=\frac 12\rbrace$. 
\end{itemize} 
 
%We provide rate of convergence better than in passive learning setting \citep{chaudhuri2014rates}

%According to the Theorem~\ref{theo1}, it is interesting to compare our rate with the best known nonparametric active rate obtained in \citep{locatelli2017adaptivity}. 

In the following, we will show that the rate of convergence of our algorithm remains the same as \eqref{rate2}, despite the use of more general hypotheses.

%$P_X$ uniform\footnote{dans un premier temps},

%   $$\tilde{O} (n^{-\alpha(\beta+1)/(2\alpha+d-\alpha\beta)}).$$

\section{KALLS algorithm}
\label{sec:algorithm}

% TODO : trouver un acronyme 
% Knn Active Learning under Local Smoothness

% TODO : bien structurer cette partie
\subsection{Setting}
%\textcolor{red}{mettre ici des indications generales sur l'algorithme??}
\label{sec:setting}

As explained in Section~\ref{sec:ALsetting}, we consider an active learning setting with a pool of i.i.d. unlabeled examples $\mathcal{K}=\lbrace X_1,X_2,\ldots, X_w \rbrace$. Let $n\leq w$ the budget, that is the maximum number of points whose label we are allowed to query to the oracle. %Recall that in a passive learning setting, we would have $n=w$. 
The objective of the algorithm is to build a 1-NN classifier, based on a labelled set $\mathcal{S}_{ac}$ of carefully chosen points. This set contains a subset of most \textit{informative} points in $\mathcal{K}$ and is called the \textit{active set}. More precisely, a point  $X_{t}$  is considered \textit{informative} if its label cannot be inferred (see below) from the previous observations $X_{t'}$(with $t'<t$).
The set $\mathcal{S}_{ac}$ starts with $X_{t_1}=X_1$ chosen arbitrarily in $\mathcal{K}$ and stops when the budget $n$ is reached or when $X_{w}$ is attained.

When a point $X_t$ is \textit{informative}, instead of requesting directly its label to the (noisy) oracle, we infer it by requesting the labels of its nearest neighbors in $\mathcal{K}$, as was done in \citep{hannekenonparametric}. This is reasonable for practical situations where the uncertainty about the label of $X_{t}$ has to be overcome, and it is related to the assumption \eqref{def:smooth}. Note that it differs from the setting of \citep{locatelli2018adaptive}, where the label of $X_{t}$ is requested several times. 
The number of neighbors $k_{t}$ used for inferring that label of $X_t$ is determined such that, while respecting the budget, we can predict with high confidence the true label as $f^*(X_{t})$ of $X_{t}$ by the empirical mean of the labels of its $k_{t}$ nearest neighbors.

The labelled active set $\mathcal{S}_{ac}$ output by the algorithm will comprise only the informative points on which we have sufficient guarantees when considering the inferred label as the right label. Finally, we show that the labelled active set $\mathcal{S}_{ac}$ is sufficient to predict the label of any new point with a 1-NN classification rule $\widehat f_{n,w}$.

\subsection{Algorithm}
\label{subsec:algorithm}

The \texttt{KALLS} algorithm (Algorithm~\ref{algo:KALLS}) aims at determining the \textit{active set} defined in Section~\ref{sec:setting} and the related 1-NN classifier $\widehat{f}_{n,w}$ under the assumption \eqref{def:smooth} and \eqref{def:TsybakovMarginNoise}. 

\noindent Before beginning the description of \texttt{KALLS}, let us introduce some variables and notations, whose precise form will be justified in Section~\ref{sec:guarantee}. The latter contains the proof sketch of the convergence of \texttt{KALLS}, while the complete proofs are in Appendix~\ref{sec:appendix}.

\noindent For $\epsilon,\delta$ $\in$ $(0,1)$, $k\geq 1$, set: 
\begin{equation}
b_{\delta,k}=\sqrt{\frac 2k\left(\log\left(\frac{1}{\delta}\right)+ \log\log\left(\frac{1}{\delta}\right)+ \log\log(e.k)\right)}.
\label{eq:pi_ks}
\end{equation}
\begin{equation}
k(\epsilon,\delta)=\frac{c}{\Delta^2}\left[\log(\frac{1}{\delta})+\log\log(\frac{1}{\delta})+\log\log\left(\frac{512\sqrt{e}}{\Delta}\right)\right].
\label{eq:k_s}
\end{equation}
where 
\begin{equation}
\Delta=\max(\frac{\epsilon}{2}, \left(\frac{\epsilon}{2C}\right)^{\frac{1}{\beta+1}}),\quad c\geq 7.10^6.
\label{eq:margin1}
\end{equation}
Let
\begin{equation}
\phi_n=\sqrt{\frac {1}{n}\left(\log\left(\frac{1}{\delta}\right)+ \log\log\left(\frac{1}{\delta}\right)\right)}.
\label{eq:sigma}
\end{equation}

% j'ai mis 2 en dehors de la racine pour assurer que loglog soit toujours positif
For $X_s$ $\in$ $\mathcal{K}=\lbrace X_1,\ldots, X_w\rbrace$, we denote henceforth by $X^{(k)}_s$ its $k$-th nearest neighbor in $\mathcal{K}$, and $Y^{(k)}_s$ the corresponding label.
%\textcolor{red}{add details here : The  \texttt{confidentLabel} subroutine aims at ...}

%\textcolor{red}{ on avait dit que $\widehat\Delta=\tilde{O}(n^{-\frac{\alpha}{2\alpha+d-\alpha\beta}})$ -- expliquer !!!}

For an integer $k\geq 1$, let 
\begin{equation}
\label{eq:eta_Xs}
\widehat{\eta}_k(X_s)=\frac 1k\sum_{i=1}^{k} Y_{s}^{(i)},\quad \bar\eta_k(X_s)=\frac 1k\sum_{i=1}^{k} \eta(X_{s}^{(i)}).
\end{equation}
%For a set $S\subset \mathcal{X}\times\mathcal{Y}$, we denote by $S_x$ the unlabelled part of $S$: 
%\begin{equation}
%\label{eq:unlabel}
%S_x=\lbrace X\in \mathcal{X},\;(X,Y)\in S\rbrace.
%\end{equation}

% TODO : labeled (US) or labelled (UK) ? choisir une des deux et s'y tenir

The inputs of \texttt{KALLS} are a pool $\mathcal{K}$ of unlabelled data of size $w$, the budget $n$, the smoothness parameters ($\alpha$, $L$)  from \eqref{def:smooth}, the margin noise parameters ($\beta$, $C$) from \eqref{def:TsybakovMarginNoise}, a confidence parameter $\delta$ $\in$ $(0,1)$ and an accuracy parameter $\epsilon$ $\in$ $(0,1)$. For the moment, these parameters are fixed from the beginning but adaptive algorithms such as \citep{locatelli2017adaptivity} could be exploited, in particular for the  $\alpha$ and $\beta$ parameters. 
%Some constants also appear in the algorithm, as discussed in Section~\ref{subsec:constants}.

At any given stage, the current version of the labelled active set $\mathcal{S}_{ac}$ is denoted by $\widehat{\mathcal{S}}$. Based on $\mathcal{S}_{ac}$,  with high confidence, the 1-NN classifier $\widehat{f}_{n,w}$ agrees with the Bayes classifier at points that lie beyond some margin $\Delta_o>0$ of the decision boundary. Formally, given $x$ $\in$ $\mathcal{X}$ such that $\vert\eta(x)-1/2\vert>\Delta_0$, we have $\displaystyle\widehat f_{n,w}(x)=\mathds{1}_{\eta(x)\geq 1/2}$ with high confidence. We will show in Section~\ref{sec:guarantee} that, with a suitable choice of $\Delta_o$, the assumption \eqref{def:TsybakovMarginNoise} leads to the desired rate of convergence \eqref{rate2}. 

\begin{algorithm}[h!]
\label{algo:KALLS}
\caption{$k$-NN Active Learning under Local Smoothness (KALLS)}
 \KwIn{a pool $\mathcal{K}=\lbrace X_1,\ldots,X_w\rbrace$,  label budget $n$, smoothness parameters ($\alpha$, $L$), margin noise parameters ($\beta$, $C$), confidence parameter $\delta$, accuracy parameter $\epsilon$.}  
\KwOut{1-NN classifier $\widehat{f}_{n,w}$}
$s=1$\Comment{index of point currently examined}\\
$\widehat{\mathcal{S}}=\emptyset$ \Comment{current active set}\\
 $t=n$ \Comment{current label budget}\\
 $I=\emptyset$\Comment{Set of informative point indexes (used for theoretical proofs)}
% $I=\emptyset$ \Comment{The set used for determining the informativeness of a point}\\ 
 %$\widehat{\Delta}=\max(\frac{\epsilon}{2}, \left(\frac{\epsilon}{2C}\right)^{\frac{1}{\beta+1}})$;\\
 
\While{$t>0$ and $s< w$}{
Let $\delta_s=\frac{\delta}{32s^2}$\\
$T$=\texttt{Reliable}($X_s$, $\delta_s$, $\alpha$, $L$, $\widehat{\mathcal{S}}$) \\
\If{T=True}{$s=s+1$}
\Else{$[\widehat{Y}_s,Q_s]$=\texttt{confidentLabel}($X_s$, $k(\epsilon,\delta_s)$, $t$, $\delta$)\\
$\displaystyle \widehat{LB}_s=\left|\frac{1}{\vert Q_{s}\vert}\sum_{(X,Y)\in Q_{s}} Y-\frac 12\right|- b_{\delta_s,\vert Q_{s}\vert}$ \Comment{Lower bound guarantee on\\ \hspace{6.5cm} $\vert\eta(X_s)-\frac 12\vert $}\\
$t=t-\vert Q_s\vert$\\
$I=I\cup \lbrace s\rbrace$\\
\If{$\widehat{LB}_s\geq 0.1b_{\delta_s,\vert Q_s\vert}$}{
$\widehat{\mathcal{S}}=\widehat{\mathcal{S}}\cup\lbrace (X_s, \widehat{Y}_s,\widehat{LB}_s)\rbrace$}
}
$s=s+1$}
$\mathcal{S}_{ac}= \lbrace (X_s,\widehat{Y}_s),\; (X_s,\widehat{Y}_s,\widehat{LB}_s)\in\widehat{\mathcal{S}}\rbrace$\\
$\widehat{f}_{n,w}\leftarrow$ \texttt{1-NN} $(\mathcal{S}_{ac})$

\end{algorithm}

%\textcolor{red}{Learn n'est decrit nulle part en tant qu'algorithme. On pourrait juste mettre $\widehat{f}_{n,w}\leftarrow$ \texttt{1NN} $(\mathcal{S}_{ac})$ ??} \textcolor{blue}{OKAY}

\texttt{KALLS} uses two main subroutines : \texttt{Reliable} and  \texttt{ConfidentLabel}, which are detailed below in Sections~\ref{sec:Reliable} and \ref{sec:ConfidentLabel}, respectively.
%It also uses a small subroutine called \texttt{Learn} to output the final 1-NN classifier $\widehat{f}_{n,w}$ (Section~\ref{sec:Learn})

\subsection{\texttt{Reliable} subroutine}
\label{sec:Reliable}
The  \texttt{Reliable} subroutine is a binary test that checks if the label of a current point $X$ can be inferred with high confidence from some previously informative points before reaching $X$. These points are obtained via a set $\widehat{\mathcal{S}}$ called \textit{current active set}. Each element of 
$\widehat{\mathcal{S}}$ can be seen as a triplet $(X',\widehat{Y}', c)$ where $X'$ is an informative point, $\widehat{Y}'$ its inferred  label, and $c>0$ can be thought  as a guarantee for predicting the right label $Y$ of $X'$ as $\widehat{Y}'$. Formally, we have $O(c)\leq \vert\eta(X')-\frac 12\vert$ when $(X',\widehat{Y}', c)$ $\in$ $\widehat{\mathcal{S}}$ and $X'$ is relatively far from the decision boundary. If \texttt{Reliable}$(X, \delta, \alpha,L,\widehat{\mathcal{S}})$ outputs \texttt{True}, the point  $X$ is not considered to be informative, and $\widehat{\mathcal{S}}$ will not be updated. By convention, \texttt{Reliable}$(X, \delta, \alpha,L,\emptyset)$ always returns $False$.

The inputs are the current point $X$, a confidence parameter $\delta$, the smoothness parameters $(\alpha,L)$ from \eqref{def:smooth}, and the set $\widehat{\mathcal{S}}$ before examining the point $X$.

%2X pareil : When  $(X',c,k)$ $\in$ $I$, and $X'$ is relatively far away from the decision boundary, we have a lower confidence bound $O(c)\leq \vert \eta(X')-\frac 12\vert$. 
If  $\vert \eta(X)-\frac 12\vert$ entails the same confidence lower bound $O(c)$ as that of some previous informative point $X'$ $($with $(X',\widehat{Y}',c)\,\in\;\widehat{\mathcal{S}})$, there is a low degree of uncertainty on the label of $X$, and $X$ is considered to be uninformative. 

Using the assumption \eqref{def:smooth}, it suffices to have 
\begin{equation}
\label{eq:reliableMin}
\min(P_X(B(X,\rho(X',X)),P_X(B(X',\rho(X',X)))\leq O(c^{d/\alpha}).
\end{equation}
Because the $P_X$ appearing in \eqref{eq:reliableMin} are unknown, it has to be replaced by an estimate. We will show  that it can be estimated with arbitrary precision and confidence using only unlabelled data from $\mathcal{K}$. %Specifically, the subroutine \texttt{BerEst} (Algorithm~\ref{algo:BerEst}) and \texttt{EstProb} (Algorithm~\ref{algo:EstProb}) will be used in the following way.
%\textcolor{blue}{c'est plus clair?}
% \textcolor{red}{adaptively??} \textcolor{blue}{On ne connait pas $P_X$, et on souhaite estimer quelque chose qui depend de $P_X$ mais sans toutefois faire appel a la connaissance de $P_X$} 
%	evaluate \eqref{eq:reliableMin} with high probability over the data.
	%$P_X(B(X_{s},\rho(X_{s'},X_s))$ $($respectively $P_X(B(X_{s'},\rho(X_{s'},X_s)))$ up to $O(c^{d/\alpha})$.

\begin{algorithm}[h!]\label{algo:Reliable}
\caption{{\texttt{Reliable} subroutine}}
 \KwIn{an instance $X$, a confidence parameter $\delta$, smoothness parameters $\alpha$,  $L$, a set $\widehat{\mathcal{S}}\subset \mathcal{X}\times\mathcal{Y}\times \mathbb{R}^+$}
 \KwOut{$T$}
 
 \For {$ (X',Y',c) \in \widehat{\mathcal{S}}$}{ 
 $\widehat{p}_{X'}=\texttt{EstProb}\left(X',\rho(X,X'),\left(\frac{c}{64L}\right)^{d/\alpha},50,\delta\right)$\\
 $\widehat{p}_{X}=\texttt{EstProb}\left(X,\rho(X,X'),\left(\frac{c}{64L}\right)^{d/\alpha},50,\delta\right)$}
\If{$\exists$ $(X',Y,c)$  $\in$ $\widehat{\mathcal{S}}$ such that $($ $\widehat{p}_{X'}\leq \frac{75}{94}\left(\frac{c}{64L}\right)^{d/\alpha} $ OR  $\widehat{p}_{X}\leq \frac{75}{94}\left(\frac{c}{64L}\right)^{d/\alpha})$ }{$T=True$}
\Else{$T=False$}
\end{algorithm}

\begin{algorithm}[h!]\label{algo:EstProb}
\caption{{\texttt{EstProb} subroutine}}
 \KwIn{an instance $x\in \mathcal{X}$, a positive number $r>0$, an accuracy parameter $\epsilon_o$, an integer parameter $u$,  a confidence parameter $\delta$}
 \KwOut{$\widehat{p}_X$ \Comment{An estimate of $P_X(B(x,r))$}}
 
Set $p\sim \mathds{1}_{B(x,r)}$ a Bernoulli variable \\
$\widehat{p}_X=\texttt{BerEst}(\epsilon_o,\delta,u)$
\Comment{ in \texttt{BerEst} subroutine, to draw a single $p_i$, randomly sample $X_i$ $\in$ $\mathcal{K}$, and set $p_i=\mathds{1}_{X_i\in B(x,r)}$.} 
\end{algorithm}

\begin{algorithm}[h!]\label{algo:BerEst}
\caption{\texttt{BerEst} subroutine (Bernoulli Estimation)}
 \KwIn{accuracy parameter $\epsilon_o$, confidence parameter $\delta'$,\\ \hspace{1cm} budget parameter $u$. \Comment{$u$ does not depend on the label budget $n$}}

\KwOut{$\widehat{p}$}
Sample $p_1,\ldots, p_4$ \Comment{with respect to $\sim p$}\\ 
$S=\lbrace p_1,\ldots, p_4\rbrace$ \\
$K=\frac{4u}{\epsilon_o}\log(\frac{8u}{\delta'\epsilon_o})$\\

\For {$i=3$\,:\, $\log_2(u\log(2K/\delta')/\epsilon_o)$}{
$m=2^i$ \\
$S=S\cup\lbrace p_{m/2+1},\ldots,p_m\rbrace$\\
$\displaystyle \widehat{p}=\frac 1 m\sum_{j=1}^{m}p_j$\\
\If{$\widehat{p}>u\log(2m/\delta')/m$}{Break}

}
Output  $\widehat{p}$

\end{algorithm}
%\textcolor{red}{This is a bit confusing for readers. Reliable calls EstProb but it is BerEst that appears as a subroutine. Can we change the text and write BerEst as Algorithm 3 ?}
%\textcolor{blue}{OK}
The \texttt{Reliable} subroutine uses \texttt{EstProb}$(X,r,\epsilon_o,50,\delta)$ (inspired from \citep{kontorovich2016active})  as follows: 
\begin{enumerate}
\item Call the subroutine \texttt{BerEst}$(\epsilon_o,\delta,50)$.
\item To draw a single $p_i$ in \texttt{BerEst}$(\epsilon_o,\delta,50)$, sample randomly an example $X_i$ from $\mathcal{K}$, and set $p_i=\mathds{1}_{X_i\in B(X,r)}$.
\end{enumerate}

 The subroutine \texttt{BerEst} consists in estimating adaptively with high probability the expectation of a Bernoulli variable $Z\sim p$. %, adapting to the unknown parameters involved in the computation of $B(X,r)$. 
 In our setting, we estimate a probability-ball, so that a realisation of $Z$ can be set as $p_i= 1_{X_i\in B(x,r)}$. The variables $p_1,\ldots, p_4$ are sampled at the beginning for theoretical analysis where we want a concentration inequality to hold for a number of samples greater than $4$ (see \citep{kontorovich2016active,maurer2009empirical} for more details).\\
However, it is not dramatic if $P_X$ is supposed to be known by the learner. This is not a limitation, since it can be assumed that the pool $\mathcal{K}$ of data is large enough such that $P_X$ can be estimated to any desired accuracy.
\subsection{\texttt{ConfidentLabel} subroutine}
\label{sec:ConfidentLabel}
If a point $X$ is considered informative, it is introduced in the \texttt{ConfidentLabel} (Algorithm\eqref{algo:ConfidentLabel}), along with an integer $k'$, a budget parameter $t$ and a confidence parameter $\delta$. 
This subroutine infers with high confidence (at least $1-\delta$) the label of $X$, by using the labels of its $k'$ nearest neighbors, knowing that we can request at most $t$ labels. The parameter $k'$ is chosen such that, with high probability, the empirical majority of the $k'$-NN labels differs from the majority in expectation by less than some margin, and all the $k'$-NN are at most at some distance from $X$. 
%\textcolor{red}{which distance ? peut-on preciser un peu ?}\textcolor{blue}{pas vraiment, sinon on rentrera trop dans les details. cela provient initialement de  \eqref{eq:rp}. il faudra ensuite expliquer la motivation de ce $r_p$}. 
The \texttt{ConfidentLabel} subroutine outputs $\widehat{Y}$, $Q$ where $Q$ represents the set of labeled nearest neighbors in the subroutine, and $\widehat{Y}$ represents the majority label in $Q$.

% \textcolor{red}{ne pas laisser des parametres non definis. Faire reference aux definitions, sans rentrer dans les details : (XS TODO : relire les autres algos et faire pareil)} \textcolor{blue}{OK}.

\begin{algorithm}[h!]\label{algo:ConfidentLabel}
\caption{{\texttt{confidentLabel} subroutine}}
 \KwIn{an instance $X$, integer $k'$, budget parameter $t\geq 1$, confidence parameter $\delta$.}
 \KwOut{$\widehat{Y}$, $Q$}
 $Q=\emptyset$ \\
 $k=1$\\
 
 %For $k\leq k_s$, set $\tau_k=\sqrt{\frac{1}{k}\log(\frac{2s^2}{\delta})}$\\
% Find in $\mathcal{P}$ the $k_s$ nearest neighbors of $X_s$: $\lbrace X^{(1)},\ldots X^{(k_s)}\rbrace$ ce n'est pas efficient   \\
 \While{$k\leq \min(k',t)$}{
 Request the label $Y^{(k)}$ of $X^{(k)}$\\$Q=Q\cup \lbrace (X^{(k)},Y^{(k)})\rbrace$\\
 \If{$\displaystyle\left\vert\frac{1}{k}\sum_{i=1}^{k}Y^{(i)}-\frac 12\right\vert>2b_{\delta,k}$}{exit \Comment{cut-off condition}\\
 \Comment{$b_{\delta,k}$ is defined in \eqref{eq:pi_ks}} }
 $k=k+1$\\ }
 
  $\displaystyle\widehat{\eta}\leftarrow\frac{1}{\vert Q \vert}\sum_{(X,Y)\in Q} Y$\\
  $\widehat{Y}=\mathds{1}_{\widehat{\eta}\geq 1/2}$\\  
\end{algorithm}

\section{Theoretical motivations}
\label{sec:guarantee}
This Section provides the main results and theoretical motivations behind the \texttt{KALLS} algorithm.  
Let us recall $\mathcal{K}=\lbrace X_1,\ldots, X_w\rbrace$ is the pool of unlabeled data and $n$ is the budget. 

 Let us denote by $\mathcal{A}_{a,w}$ the set of active learning algorithms on $\mathcal{K}$, and $\mathcal{P}(\alpha, \beta)$ the set of probabilities that satisfy assumption \eqref{def:smooth} and \eqref{def:TsybakovMarginNoise}. 

 Additionally, let us introduce the set of probabilities $\mathcal{P}'(\alpha, \beta)$ on $\mathcal{X}\times\mathcal{Y}$. A probability $P$ $\in$ $\mathcal{P}'(\alpha, \beta)$ if $P$ $\in$ $\mathcal{P}(\alpha, \beta)$ and its marginal probability $P_X$ is a doubling-probability. For $A$ $\in$ $\mathcal{A}_{a,w}$, we denote by $\widehat{f}_{A,n,w}:=\widehat{f}_{n,w}$ the classifier that is provided by $A$.

Theorem~\ref{upperbound} and its equivalent form in Theorem~\ref{upperboundcomplexity} are the main results of this paper. They provide bounds on the excess risk for the \texttt{KALLS} algorithm in terms of the set $\mathcal{P}'(\alpha,\beta)$. 
The main idea of the proof is sketched in Section~\ref{sec:mainIdea}, while a detailed proof can be found in Appendix~\ref{sec:appendix}.

\subsection{Main results}
\label{sec:mainIdea}
 
%introduces three known lemmas that we be used in Subsection~\ref{sec:coreTheorems}, that contains the new results of this paper, that allow to proof   that the algorithm KALLS achieves the guarantee  given in Theorem~\ref{upperboundcomplexity}.

%For a classifier $\widehat{f}_{n,w}$, it is well known \citep{lugosi2002pattern} that the excess of risk is:
%\begin{equation}
%R(\widehat{f}_{n,w})-R(f^*)=\int_{\lbrace x,\,\widehat{f}_{n,w}(x)\neq f^*(x)\rbrace} \vert 2\eta(x)-1\vert dP_X(x).
%\label{eq:excess-risk-def}
%\end{equation}

\begin{theorem}[Excess risk for the \texttt{KALLS} algorithm.]\label{upperbound}~\\
Let the set $\mathcal{P}'(\alpha, \beta)$ such that $\alpha\beta<d$ where $d$ is the dimension of the input space $\mathcal{X}\subset \mathbb{R}^d$. For $P$ $\in$ $\mathcal{P}'(\alpha, \beta)$, if $\widehat{f}_{n,w}$ is the $1-NN$ classifier provided by \texttt{KALLS}, then we have:
\begin{equation}
\sup_{P\in\mathcal{P}'(\alpha,\beta)}\,\mathbb{E}_n\left[ R(\widehat{f}_{n,w})-R(f^*)\right]\leq \tilde{O}\left(n^{-\frac{\alpha(\beta+1)}{2\alpha+d-\alpha\beta}}\right),
\label{eq:excess-risk}
\end{equation}
where $\mathbb{E}_n$ is  with respect to the randomness of the \texttt{KALLS} algorithm.
\end{theorem}
The result \eqref{eq:excess-risk} is also stated below (Theorem~\ref{upperboundcomplexity}) in a more practical form using label complexity. 
This latter form will be used in the proof.
\begin{theorem}[Label complexity for the \texttt{KALLS} algorithm.]~\\
\label{upperboundcomplexity}
Let the set $\mathcal{P}'(\alpha, \beta)$ such that $\alpha\beta<d$. Let $\epsilon$, $\delta$ $\in$ $(0,1)$. For all $n$, $w$ $\in$ $\mathbb{N}$ such that: \\ 

\textbf{if} 
\begin{equation}
\label{eq:label-complexity}
n\geq \tilde{O}\left( \left(\frac {1}{\epsilon}\right)^{\frac{2\alpha+d-\alpha\beta}{\alpha(\beta+1)}}\right),
\end{equation}

\begin{equation}
\label{guarantee-on-pool-0}
w\geq \tilde{O}\left( \left(\frac{1}{\epsilon}\right)^{\frac{2\alpha+d}{\alpha(\beta+1)}}\right)
\end{equation}

and
\begin{equation}
w\geq \frac{400\log\left(\frac{12800w^2}{\delta (\frac{1}{64L}\bar c\phi_n)^{d/\alpha}}\right)}{(\frac{1}{64L}\bar c\phi_n)^{d/\alpha}},
\label{condition2-0}
\end{equation}
where $L$ appears in \eqref{def:smooth}, $\bar c=0.1$ and $\phi_n$ is defined by \eqref{eq:sigma},
%il faut prendre n tres petit devant w

\textbf{then} with probability at least $1-\delta$ we have: 
\begin{equation}
\label{eq:error}
\sup_{P\in\mathcal{P}'(\alpha,\beta)}\,\left[ R(\widehat{f}_{n,w})-R(f^*)\right]\leq \epsilon.
\end{equation}
\end{theorem}

\noindent Before proving this theorem, a couple of important remarks should be made:
\begin{enumerate}
\item The rate of convergence \eqref{eq:excess-risk} obtained in Theorem~\ref{upperbound}  is an improvement over the passive learning counterpart. For $P$ $\in$ $\mathcal{P}(\alpha, \beta)$, if $\widehat{f}_{n}$ is the classifier provided by a passive learning algorithm, we have \citep{chaudhuri2014rates}:
\begin{equation}
\sup_{P\in\mathcal{P}(\alpha,\beta)}\,\mathbb{E}_n\left[ R(\widehat{f}_{n})-R(f^*)\right]\leq \tilde{O}\left(n^{-\frac{\alpha(\beta+1)}{2\alpha+d}}\right). 
\end{equation}
Because $\mathcal{P}'(\alpha,\beta)\subset \mathcal{P}(\alpha,\beta)$, we also have: 
\begin{equation}
\sup_{P\in\mathcal{P}'(\alpha,\beta)}\,\mathbb{E}_n\left[ R(\widehat{f}_{n})-R(f^*)\right]\leq \tilde{O}\left(n^{-\frac{\alpha(\beta+1)}{2\alpha+d}}\right). 
\end{equation}

\item The rate \eqref{eq:excess-risk} is also minimax.  Indeed, let us introduce the set of probabilities  $\bar{\mathcal{P}}(\alpha,\beta)$ that satisfy the H\"older continuous assumption \eqref{def:Holder} (with parameter $\alpha$),  strong density assumption \eqref{def:strongDensity}, and margin noise assumption \eqref{def:TsybakovMarginNoise}.
%\textcolor{red}{je trouve bizarre de melanger les notations H1, H2, avec Assumption 2. Je mettrais plutot \eqref{def:TsybakovMarginNoise}. D'ailleurs, est-ce vraiment indispensable d'utiliser les notations "Assumption" ailleurs ?} \textcolor{blue}{OKAY}

Let us assume that $\alpha\beta< d$. It was proven in \citep{minsker2012plug} that if $supp(P_X)\subset [0,1]^d$, there exists a constant $\gamma>0$ such that for all $n$ large enough and for any active classifier $\widehat{f}_n$, we have: 
\begin{equation}
\label{eq:lowerrate}
\sup_{P\in\bar{\mathcal{P}}(\alpha,\beta)}\,\left[ R(\widehat{f}_{n})-R(f^*)\right]\geq \gamma n^{-\frac{\alpha(\beta+1)}{2\alpha+d-\alpha\beta}}.
\end{equation}
Moreover, the strong density assumption implies the doubling-probability assumption \citep{kpotufe2011k}, and according to Theorem~\ref{theo1}, the lower bound obtained in \eqref{eq:lowerrate} is also valid for the family of probabilities $\mathcal{P}'(\alpha,\beta)$.
 \end{enumerate}

\subsection{Proof sketch of Theorem \eqref{upperboundcomplexity}}
\label{sec:mainIdea}
For a classifier $\widehat{f}_{n,w}$, it is well known\citep{lugosi2002pattern} that the excess of risk is:
\begin{equation}
\label{excess-risk}
R(\widehat{f}_{n,w})-R(f^*)=\int_{\lbrace x,\,\widehat{f}_{n,w}(x)\neq f^*(x)\rbrace} \vert 2\eta(x)-1\vert dP_X(x).
\end{equation}
We thus aim to prove that \eqref{eq:label-complexity} is a sufficient condition to guarantee (with probability $\geq$ $1-\delta$), that $\widehat{f}_{n,w}$ agrees with $f^*$ on the set $\lbrace x,\; \vert\eta(x)-1/2\vert> \Delta_o\rbrace$, for a suitable choice of $\Delta_o>0$.

 Introducing $\Delta_o$ in \eqref{excess-risk} leads to:  
\begin{equation}
R(\widehat{f}_{n,w})-R(f^*) \leq 2\Delta_o P_X( \vert \eta(x)-1/2\vert<\Delta_o).
\end{equation}
Therefore, if $\Delta_o\leq \frac{\epsilon}{2}$ then we have immediately, $R(\widehat{f}_{n,w})-R(f^*)\leq \epsilon$. On the other hand, if $\Delta_o> \frac{\epsilon}{2}$, by hypothesis \eqref{def:TsybakovMarginNoise}, we have $R(\widehat{f}_{n,w})-R(f^*)\leq 2C\Delta_o^{\beta+1}.$ In the latter case, setting $\Delta_o=\left(\frac{\epsilon}{2C}\right)^{\frac{1}{\beta+1}}$  guarantees $R(\widehat{f}_{n,w})-R(f^*)\leq \epsilon$. 
Altogether, using for $\Delta_o$ the value $\Delta=\max(\frac{\epsilon}{2},\left(\frac{\epsilon}{2C}\right)^{\frac{1}{\beta+1}})$  guarantees $R(\widehat{f}_{n,w})-R(f^*)\leq \epsilon$. This explains the expression \eqref{eq:margin1}.\medskip

We present the proof sketch of Theorem \ref{upperboundcomplexity} in three main steps, and refer to the corresponding Lemmas and Theorems in the Appendix~\ref{sec:appendix} for more details.

\begin{enumerate}
\item \underline{\textbf{Adaptive label requests on informative points:}}
\medskip

%\textcolor{red}{do we specify $A_1$ and $A_2$ at some point ? Dans "Consistency of Nearest Neighbor Classification under Selective Sampling", Dasgupta utilise aussi une presentation avec des events $E_0$, $E_1$, ... mais c'est specifie (on peut en parler, peut-etre que je ne suis juste pas habitue a cette maniere de presenter)}
We design two events $A_1, A_2$  with $P(A_1\cap A_2)\geq  1-\frac{3\delta}{16}$, such that: 

\begin{itemize}
\item[$\bullet$] Given an informative point $X_s$, 
%the following relations hold on $A_1\cap A_2$ for all $k\geq 1$: 
%\begin{equation}
%\vert\widehat \eta_k(X_s)-\bar\eta_k(X_s)\vert\leq \sqrt{\frac{2\log(\frac{32s^2}{\delta})}{k}} \quad \text{(see Lemma \ref{lemma:subgaussian})}
%\label{eq:sketch-hoeffding}
%\end{equation}
%and 
%\begin{equation}
%\vert\widehat \eta_k(X_s)-\bar\eta_k(X_s)\vert\leq b_{\delta,k} \quad \text{(see Lemma \ref{theorem:savings label})}.
%\label{eq:sketch-Kaufmann}
%\end{equation}
%
%The equations \eqref{eq:sketch-hoeffding} and \eqref{eq:sketch-Kaufmann} play different roles: \\
if $\vert \eta(X_s)-\frac12 \vert\geq \frac 12\Delta$ and if the budget allows $(n=w=+\infty)$, on $A_1\cap A_2$, the cut-off condition used in Algorithm \ref{algo:ConfidentLabel} $$\left\vert\frac{1}{k}\sum_{i=1}^{k}Y^{(i)}_s-\frac 12\right\vert\leq 2b_{\delta_s,k}$$  
 will be violated after at most $\tilde k(\epsilon,\delta_s)$ requests, with $\tilde k(\epsilon,\delta_s)\leq k(\epsilon,\delta_s)$. Also, the label inferred after $\tilde k(\epsilon,\delta_s)$ label requests corresponds to the true label $f^*(X_s)$. The intuition behind is to adapt the number of labels requested with respect to the noise; i.e., fewer label requests on a less noisy point $(\text{i.e.,}\,\vert \eta(X_s)-\frac12 \vert\geq \frac 12\Delta)$, and more label requests on a noisy point. This provides significant savings in the number of requests needed to predict with high probability the correct label.
   
\item[$\bullet$] In the event $A_1\cap A_2$,  any informative point $X_s$ falls in a high density region such that all the $k(\epsilon,\delta_s)$ nearest neighbors of $X_s$ are within at most some distance to $X_s$, and the condition \eqref{guarantee-on-pool-0} is sufficient to have $k(\epsilon,\delta_s)\leq w$. 
\end{itemize}
\medskip
\item \underline{\textbf{Condition to be an informative point}}
\medskip

We design an event $A_3$ with $P(A_3)\geq 1-\delta/16$ and such that, on this event $A_3$, for $X_s$ $(s\leq w)$ a point whose informativeness we want to check, and $\widehat{S}$ the version of the current active set  just before reaching the point $X_s$ in \texttt{KALLS} (Algorithm \eqref{algo:KALLS}), the following holds.

\textbf{If} there exists  $s'<s$ such that $X_{s'}$ is an informative point and $(X_{s'},\widehat{Y}_{s'},\widehat{LB}_{s'})$ $\in$ $\widehat{S}$, and that satisfies 
\begin{equation}
\left(\widehat{p}_{X_{s'}}\leq \frac{75}{94}\left(\frac{1}{64L}\widehat{LB}_{s'}\right)^{d/\alpha}\,\text{or}\; \widehat{p}_{X_{s}}\leq \frac{75}{94}\left(\frac{1}{64L}\widehat{LB}_{s'}\right)^{d/\alpha}\right)
\label{eq:sketch-lower-bound}
\end{equation}
\hspace{1cm} where $\widehat{p}_{X_{s'}}$ and $\widehat{p}_{X_{s'}}$ are defined in Algorithm\ref{algo:Reliable},\\
\textbf{then},
\begin{equation}
\min (P_X(B(X_s,\rho(X_s,X_{s'}))),P_X(B(X_{s'},\rho(X_{s},X_{s'}))))\leq \left(\frac{1}{64L}\widehat{LB}_{s'}\right)^{d/\alpha}.
\label{eq:sketch-estimate}
\end{equation} 
In this case, let $X_{s'}$ be such a point that satisfies \eqref{eq:sketch-lower-bound} and \eqref{eq:sketch-estimate}, we can easily prove that when $X_{s'}$ is relatively far from the boundary, i.e., $\vert \eta(X_{s'})-\frac 12\vert\geq \frac{\Delta}{32}$, on $A_1\cap A_2\cap A_3$,  we have the lower bound guarantee 
\begin{equation}
\vert \eta(X_{s'})-\frac 12\vert\geq \frac{32}{63}\widehat{LB}_{s'},
\label{eq:sketch-lower-guarantee}
\end{equation} 
and easily deduce by using the smoothness assumption,  \eqref{eq:sketch-estimate} and \eqref{eq:sketch-lower-guarantee}, that the points $X_{s}$ and $X_{s'}$ have the same label, then we do not need to use $X_s$ in the subroutine \texttt{ConfidentLabel} (Algorithm \ref{algo:ConfidentLabel}) and $X_s$ is an uninformative point. In addition, \eqref{condition2-0} is a sufficient condition such that the number of points used in $\texttt{Estprob}(X_s,\rho(X_s,X_{s'}), \left(\frac{1}{64L}\widehat{LB}_{s'}\right)^{d/\alpha},50,\delta_s)$ (respectively in $\texttt{Estprob}(X_{s'},\rho(X_s,X_{s'}), \left(\frac{1}{64L}\widehat{LB}_{s'}\right)^{d/\alpha},50,\delta_s)$) is lower than $w$.   
%for some previous informative point $X_s$ (with $\widehat{LB}_{s}>0$ defined in Lemma \ref{lemma:sufficient-informative}). Because $P_X$ is unknown, we provide a computational scheme sufficient to obtain ~\eqref{equation:suffi_uninformativ}.\\Firstly we follow the general procedure used in \citep{kontorovich2016active} to estimate adaptively the expectation of a Bernoulli random variable. And secondly, we apply it to the Bernoulli variable $\mathds{1}_{A}$ where $A=\lbrace x,\;x\in B(X_t,r)\rbrace$.
\medskip

\item \underline{\textbf{Label the instance space and label complexity}}
\medskip

The set $I$ is introduced in \texttt{KALLS} (Algorithm\eqref{algo:KALLS}) as the set of informative points indexes. Let $s_I= \max I$, the index of the last informative point. \\For $\epsilon, \delta$ $\in$ $(0,\frac 12)$, $\Delta$ defined in \eqref{eq:margin1}, and $(\alpha, L)$ the smoothness parameters, let us introduce
$$T_{\epsilon,\delta}=\frac{1}{\tilde p_{\epsilon}}\ln(\frac{8}{\delta}),\,\text{and} \; \tilde p_{\epsilon}=\left(\frac{\Delta}{128L}\right)^{d/\alpha}.$$ 
We design two events $A_4$ and $A_5$, with $P(A_4\cap A_5)\geq 1-\delta/4$, such that on $A_1\cap A_2\cap A_3\cap A_4$,   if
\begin{equation}
s_I \geq T_{\epsilon,\delta}
\label{eq:label-sketch-nec}
\end{equation}
and equations \eqref{condition2-0} and \eqref{guarantee-on-pool-0} hold, then, for all $x$ $\in$ $supp(P_X)$ with $\vert \eta(x)-\frac 12\vert>\Delta$, we have:   
\begin{equation}
\vert\eta(X^{(1)}_x)-\frac 12\vert\geq \frac 12\Delta\quad \text{and}\quad\widehat{f}_{n,w}(x)=f^*(x)=f^*(X^{(1)}_x),
\end{equation}
where $X^{(1)}_x$ is the nearest neighbor of $x$ in $\widehat{S}_{ac}$, and $\widehat{f}_{n,w}$ the 1-NN classifier on $\widehat{S}_{ac}$. Additionally, on $A_1\cap A_2\cap A_3\cap A_4\cap A_5$ we  prove that \eqref{eq:label-complexity} is sufficient to obtain  \eqref{eq:label-sketch-nec}.

\end{enumerate}
Finally, if  \eqref{eq:label-complexity}, \eqref{guarantee-on-pool-0} and \eqref{condition2-0} hold simultaneously, then, on $A_1\cap A_2\cap A_3\cap A_4\cap A_5$, the final classifier $\widehat{f}_{n,w}$ agrees with the Bayes classifier $f^*$ on $\lbrace x,\;\vert\eta(x)-1/2\vert> \Delta\rbrace$. Thus, \eqref{eq:error} holds with probability at least $1-\left(\frac{\delta}{16}+\frac{\delta}{8}+\frac{\delta}{16}+\frac{\delta}{8}+\frac{\delta}{8}\right)=1-\delta/2>1-\delta$. 

\section{Conclusion and future work}
\label{sec:conclusion} 
In this paper we first reviewed the main results for convergence rates in a nonparametric setting for active learning, with a special emphasis on the relative merits of the assumptions about the smoothness and the margin noise. Then, by putting active learning in perspective with recent work on passive learning, we provided a novel active learning algorithm with a particular smoothness assumption customized for $k$-NN.

We showed that our algorithm has a convergence rate comparable to state-of-the art active learning algorithms, but using less restrictive assumptions.  This removes unnecessary restrictions on the distribution that would exclude important densities (e.g., Gaussian).

Additionally, our algorithm can readily be extended to multi-class classification, and then compared to recent results obtained in passive learning by \citep{reeve2017minimax} which extended the work of \citep{chaudhuri2014rates}  to multi-class classification.

Finally, an important direction for further work is to extend our results to the case where the key parameters of the problem (defining smoothness and noise) are unknown. Ongoing work in this direction builds upon previous results in an  adaptive setting \citep{locatelli2017adaptivity}, \citep{minsker2012plug}, \citep{balcan2012robust}, \citep{hanneke2011rates}.
 
\newpage
%Practical implementations of the \texttt{KALLS} algorithm are underway.
\bibliographystyle{plainnat}      % mathematics and physical sciences
\bibliography{journal_springer}
\newpage
\appendix
\section{Detailed proof of Theorem~\ref{upperboundcomplexity}}
\label{sec:appendix}

This Appendix is organized as follows:   in Section~\ref{sec:notations}, we introduce some additional notations. In Section~\ref{subsec:adaptive}  we adaptively determine the number of label requests needed to accurately predict the label of an informative point that is relatively far from the boundary decision. In Section~\ref{subsec: suff-cond-informativ}, we provide  some lemmas that give a sufficient condition for a point to be informative. In Section~\ref{subsec:lab-instance-space}, we give theorems that allow us to classify each instance relatively far from the decision boundary. Finally in Section~\ref{subsec:label complexity}, we provide the label complexity and establish Theorem~\ref{upperboundcomplexity}.

\subsection{Notations}
\label{sec:notations}

Some notations that will be used throughout the proofs are listed here for convenience.

As defined in Section~\ref{sec:notation-def}, let $B(x,r)=\lbrace x'\in \mathcal{X},\; \rho(x,x')< r\rbrace$ and $\bar B(x,r)=\lbrace x'\in \mathcal{X},\; \rho(x,x')\leq r\rbrace $  the open and closed balls with respect to the Euclidean metric $\rho$, respectively, centered at $x$ $\in$ $\mathcal{X}$ with radius $r>0$. Let $\text{supp}(P_X)=\lbrace x\in \mathcal{X},\,\;\forall r>0,\;P_X(B(x,r))>0\rbrace$  the support of the marginal distribution $P_X$.

For $p$ $\in$ $(0,1]$, and $x$ $\in$ $supp(P_X)$,  let us define 
\begin{equation}
r_p(x)=\inf \lbrace r>0, P_X(B(x,r))\geq p\rbrace.
\label{eq:rp}
\end{equation}
Let us recall for $X_s$ $\in$ $\mathcal{K}=\lbrace X_1,\ldots, X_w\rbrace$, we denote  by $X^{(k)}_s$ its $k$-th nearest neighbor in $\mathcal{K}$, and $Y^{(k)}_s$ the corresponding label.

%\textcolor{red}{ on avait dit que $\widehat\Delta=\tilde{O}(n^{-\frac{\alpha}{2\alpha+d-\alpha\beta}})$ -- expliquer !!!}

For an integer $k\geq 1$, let 
\begin{equation}
\label{eq:eta_Xs}
\widehat{\eta}_k(X_s)=\frac 1k\sum_{i=1}^{k} Y_{s}^{(i)},\quad \bar\eta_k(X_s)=\frac 1k\sum_{i=1}^{k} \eta(X_{s}^{(i)}).
\end{equation}
%\subsubsection{Lemmas}
\subsection{Adaptive label requests on informative points }
\label{subsec:adaptive}
\begin{lemma}[Chernoff bounds, \citep{mulzer2018five}]\label{Chernoff}~\\
Suppose $X_1,\ldots, X_m$ are independent random variables taking value in $\lbrace 0,1\rbrace$. Let $X$ denote their sum and $\mu=E(X)$ its expected value. Then,
\begin{itemize}
\item For any $\delta\in (0,1)$, 
\begin{equation}
P_m(X\leq (1-\delta)\mu)\leq \exp(-\delta^2\mu/2),
\label{eq:chernoff1}
\end{equation}
where $P_m$ is the probability with respect to the sample $X_1,\ldots,X_m$.
\item Additionally, for any $\delta'\geq 1$, we have:
\begin{equation}
P_m(X\geq (1+\delta')\mu)\leq \exp(-\delta'\mu/4).
\label{eq:chernoff2}
\end{equation}
\end{itemize}
\end{lemma}

\begin{lemma}[Logarithmic relationship, \citep{vidyasagar2013learning}] 
\label{lemme-guarantee}~\\
Suppose $a,b,c>0$, $abe^{c/a}>4\log_2(e)$, and $u\geq 1$. Then: 
$$u\geq 2c+2a\log(ab)\Rightarrow \; u>c+a\log(bu).$$
\end{lemma}

\begin{lemma}\citep{chaudhuri2014rates}\label{lemme-chaudhuri_lemma}~\\
For $p$ $\in$ $(0,1]$, and $x$ $\in$ $supp(P_X)$,  let us define $r_p(x)=\inf \lbrace r>0, P_X(B(x,r))\geq p\rbrace.$
For all $p$ $\in$ $(0,1]$, and $x$ $\in$ $supp(P_X)$, we have: $$P_X(B(x,r_p(x))\geq p.$$  
\end{lemma}

%\subsubsection{Core theorems}
\label{sec:coreTheorems}

\begin{theorem}\label{theo:passive-guarantee}~\\
Let $\epsilon, \delta$ $\in$ $(0,1)$. Set $\Delta=\max(\epsilon,\left(\frac{\epsilon}{2C}\right)^{\frac{1}{\beta+1}})$, and $p_{\epsilon}=\left(\frac{31\Delta}{1024L}\right)^{d/\alpha}$, where $\alpha$, $L$, $\beta$, $C$ are parameters used in \eqref{def:TsybakovMarginNoise} and \eqref{def:smooth}.\\ 
For $p$ $\in$ $(0,1]$, and $x$ $\in$ $supp(P_X)$,  let us introduce  $r_p(x)=\inf \lbrace r>0, P_X(B(x,r))\geq p\rbrace$ and $k_s:=k(\epsilon,\delta_s)$ defined in \eqref{eq:k_s} $($where $\delta_s=\frac{\delta}{32s^2})$.
%$$k_{s}(\epsilon,\delta)=\frac{\widehat{c}}{\widehat{\Delta}^2}\left[\log(\frac{4c_1s^2}{\delta})+\log\log(\frac{4c_1s^2}{\delta})+\log\log\left(\frac{2\sqrt{e}c_5^2}{c_0c_4\widehat{\Delta}}\right)\right],$$
%where $\widehat{c},c_5,c_0,c_1$ are defined in section ~\ref{subsec:constants}
\\ For $k,s\geq 1$, set $\tau_{k,s}=\sqrt{\frac{2}{k}\log(\frac{32s^2}{\delta})}$.
There exists an event $A_1$ with probability at least $1-\frac{\delta}{16}$, such that on $A_1$, for all $1\leq s\leq w$, if 
\begin{equation}
\label{first necessary-guarantee} 
k_s\leq(1-\tau_{k_s,s})p_{\epsilon}(w-1)
\end{equation}
then  the $k_s$ nearest neighbors of $X_s$ (in the pool $\mathcal{K}$) belong to the ball $B(X_s, r_{p_{\epsilon}}(X_s))$. Additionally, the condition  
\begin{equation}
\label{guarantee-on-pool}
w\geq \tilde{O}\left( \left(\frac{1}{\epsilon}\right)^{\frac{2\alpha+d}{\alpha(\beta+1)}}\right)
\end{equation}
is sufficient to have \eqref{first necessary-guarantee}.
\end{theorem}

\begin{proof}~\\
Fix $x$ $\in$ $supp(P_X)$. For $k$ $\in$ $\mathbb{N}$, let us denote $X^{(k)}_x$, the $k^{th}$ nearest neighbor of $x$ in the pool. we have, 

$$ P(\rho(x, X^{(k_s+1)}_x)>r_{p_{\epsilon}}(x)) \leq   P(\sum_{i=1}^{w}\mathds{1}_{X_i\in B(x,r_{p_{\epsilon}}(x))}\leq k_s).$$

%Set $\bar k_s=(1-\tau_{k_s,s})P_X(B(x,r_{p_{\epsilon}}(x)) (w-1)).$
Then, by using Lemma~\ref{Chernoff} and Lemma~\ref{lemme-chaudhuri_lemma}, and if $k_s$ satisfies \eqref{first necessary-guarantee}, we have:
\begin{align*}
 P(\rho(x, X^{(k_s+1)}_x)>r_{p_{\epsilon}}(x)) 
                                                & \leq P(\sum_{i=1}^{w}\mathds{1}_{X_i\in B(x,r_{p_{\epsilon}}(x))}\leq(1-\tau_{k_s,s})p_{\epsilon} (w-1))\\
                                                &\leq P\left(\sum_{i=1}^{w}\mathds{1}_{X_i\in B(x,r_{p_{\epsilon}}(x))}\leq(1-\tau_{k_s,s})P_X(B(x,r_{p_{\epsilon}}(x))) (w-1)\right)\\
%                                              & \leq P(\sum_{i=1}^{w}\mathds{1}_{X_i\in B(x,r_{p_{\epsilon}}(x))}\leq \bar % k_s)\\ 
                                                & \leq \exp(-\tau_{k_s,s}^2(w-1)P_X(B(x,r_{p_{\epsilon}}(x))/2)\\
                                                & \leq \exp(-\tau_{k_s,s}^2(w-1)p_{\epsilon}/2)\\
                                                &  \leq \exp(-\tau_{k_s,s}^2 k_s/2)\\
                                                & \leq \exp(-\log(32 s^2/\delta))\\
                                                & = \frac{\delta}{32 s^2}.\\
\end{align*} 
Fix $x=X_s$. Given $X_s$, there exists an event $A_{1,s}$, such that $P(A_{1,s})\geq 1-\delta/(32s^2)$, and on $A_{1,s}$, if $$k_s\leq(1-\tau_{k_s,s})p_{\epsilon}(w-1),$$ we have $B(X_s,r_{p_{\epsilon}}(X_s))\cap \lbrace X_1,\ldots, X_w\rbrace \geq k_s$. By setting $\displaystyle{A_1= \cap_{s\geq 1} A_{1,s}}$, we have $P(A_1)\geq 1-\delta/16$, and on $A_1$,  for all $1\leq s\leq w$, if $k_s\leq(1-\tau_{k_s,s})p_{\epsilon}(w-1)$, then $B(X_s,r_{p_{\epsilon}}(X_s))\cap \lbrace X_1,\ldots, X_w\rbrace \geq k_s$.

Now, let us proof that the condition \eqref{guarantee-on-pool} is sufficient to guarantee \eqref{first necessary-guarantee}. \\ The relation \eqref{first necessary-guarantee} implies 
 \begin{equation}
 w\geq \frac{k_s}{(1-\tau_{k_s,s})p_{\epsilon}}+1.
\end{equation}
 We can see by a bit of calculus,  that $\tau_{k_s,s}\leq \frac{1}{2}$, and then 
 \begin{align*}
 \frac{k_s}{(1-\tau_{k_s,s})p_{\epsilon}}+1 &\leq \frac{2k_s}{p_{\epsilon}}+1\\
                                         & \leq 4\frac{k_s}{p_{\epsilon}}\qquad\left(\text{because}\, \frac{k_s}{p_{\epsilon}}\geq  1\right)\\
                                         &= \frac{4c}{p_{\epsilon}\Delta^2}\left[\log(\frac{32s^2}{\delta})+\log\log(\frac{32s^2}{\delta})+\log\log\left(\frac{512\sqrt{e}}{\Delta}\right)\right]\\
                                         &= \frac{b}{\Delta^{2+\frac{d}{\alpha}}}\left[\log(\frac{32s^2}{\delta})+\log\log(\frac{32s^2}{\delta})+\log\log\left(\frac{512\sqrt{e}}{\Delta}\right)\right],                          
\end{align*}
where $b=4c\left(\frac{1024L}{31}\right)^{d/\alpha}$. 
\begin{align*} 
\frac{k_s}{(1-\tau_{k_s,s})p_{\epsilon}}+1 &\leq \bar{C}\left(\frac{1}{\epsilon}\right)^{\frac{2\alpha+d}{\alpha(\beta+1)}}\left[\log(\frac{32s^2}{\delta})+\log\log(\frac{32s^2}{\delta})+\log\log\left(\frac{512\sqrt{e}}{\Delta}\right)\right]\\                                        
 &\quad\quad \text{as}\; \Delta=\max(\epsilon,\left(\frac{\epsilon}{2C}\right)^{\frac{1}{\beta+1}}), \;\text{where}\;\bar{C}=b(2C)^{\frac{2\alpha+d}{\alpha(\beta+1)}}\\                                        
                                         &\leq \bar{C}\left(\frac{1}{\epsilon}\right)^{\frac{2\alpha+d}{\alpha(\beta+1)}}\left[2\log(\frac{32s^2}{\delta})+\log\left(\frac{512\sqrt{e}}{\epsilon}\right)\right]\\
                                         & \quad\quad\;\text{as}\log(x)\leq x,\;\text{and}\;\Delta\geq \epsilon\\
                                         &\leq 2\bar{C}\left(\frac{1}{\epsilon}\right)^{\frac{2\alpha+d}{\alpha(\beta+1)}}\left[\log(s^2)+\log\left(\frac{16384\sqrt{e}}{\delta\epsilon}\right)\right]\\
                                         &\leq 4\bar{C}\left(\frac{1}{\epsilon}\right)^{\frac{2\alpha+d}{\alpha(\beta+1)}}\left[\log(s)+\log\left(\frac{16384\sqrt{e}}{\delta\epsilon}\right)\right]\\
                                         &\leq 4\bar{C}\left(\frac{1}{\epsilon}\right)^{\frac{2\alpha+d}{\alpha(\beta+1)}}\left[\log(w)+\log\left(\frac{16384\sqrt{e}}{\delta\epsilon}\right)\right].\\
                                          \numberthis \label{sufficient-garantee}                                        
\end{align*}
Now, we are going to apply the Lemma~\ref{lemme-guarantee}.  If we set in Lemma~\ref{lemme-guarantee} 
$$a=4\bar{C}\left(\frac{1}{\epsilon}\right)^{\frac{2\alpha+d}{\alpha(\beta+1)}}, \quad c=4\bar{C}\left(\frac{1}{\epsilon}\right)^{\frac{2\alpha+d}{\alpha(\beta+1)}}\log\left(\frac{16384\sqrt{e}}{\delta\epsilon}\right),\quad b=1$$
we can easily see that $c\geq a$, $a\geq 4$ and then
$$abe^{c/a}\geq 4e>\log_2(e).$$
Then, the  relation $$w\geq 4\bar{C}\left(\frac{1}{\epsilon}\right)^{\frac{2\alpha+d}{\alpha(\beta+1)}}\left(\log\left(\frac{16384\sqrt{e}}{\delta\epsilon}\right)+\log\left(4\bar{C}\left(\frac{1}{\epsilon}\right)^{\frac{2\alpha+d}{\alpha(\beta+1)}}\right)\right)$$ is sufficient to guarantee \eqref{sufficient-garantee}.
 %essaies de modifier le log dans k_s par Log= max(log(x),1) pour que k>=1
\end{proof}

Let us note that the guarantee obtained in the preceding theorem corresponds to that obtained in passive setting ($w=n$). %Furthermore, when \eqref{guarantee-on-pool} holds, for all $X$ $\in$ $\mathcal{K}$, all 

%The next result is inspired conjointly by results in \citep{chaudhuri2014rates}, \citep{reeve2017minimax}, \citep{hannekenonparametric}.
\subsection{Motivation for choosing $k_s$ for $X_s$}
%In the following, we give a justification for the choice of $k_s:=k(\epsilon,\delta_s)$ ~\eqref{eq:k_s}. We also prove that, if the budget permits, under a margin condition, the cut-off condition used in Algorithm \ref{algo:ConfidentLabel} $$\left\vert\frac{1}{k}\sum_{i=1}^{k}Y^{(i)}_s-\frac 12\right\vert\leq 2b_{\delta_s,k}$$  
% will be violated with a number of label requests lower than $k(\epsilon,\delta_s)$. The intuition behind this, is with Theorem~\ref{theorem:savings label}, to adapt (with respect to the noise) the number of labels requested; i.e., fewer label requests on a less-noisy point, and more label requests on a noisy point. This provides significant savings in the number of requests needed to predict with high probability the right label. 

%\textcolor{red}{ici tu parles d'intuition, ... ca me semble plus approprie dans le texte, j'essaierais de mettre en appendice vraiment les details des preuves que les gens n'ont pas forcement envie de lire. Par contre, une explication intuitive est la bienvenue dans le texte}

\begin{lemma}[Hoeffding's inequality,\citep{hoeffding1963probability}]\label{lemma:subgaussian}~\\
\begin{itemize}
\item \textbf{First version}: \\Let $X$ be a random variable with $E(X)=0$, $a\leq X\leq b$, then for $v>0$, 
$$E(e^{vX})\leq e^{v^2(b-a)^2/8}.$$
\item  \textbf{Second version}: \\
Let $X_1, \ldots, X_m$ be independent random variables such that $-1 \leq X_i \leq 1$, $(i=0,\ldots, m)$. We define the empirical mean of these variables by

    $$\bar X = \frac 1 m\sum_{i=1}^{m} X_i.$$  Then we have: 
    $$P( \vert \bar X- E(\bar X)\vert\geq t)\leq \exp(-mt^2/2)$$
\end{itemize}
\end{lemma}

\begin{lemma}\citep{kaufmann2016complexity}\label{lemma:Kaufmann16}~\\
Let $\displaystyle\zeta(u)=\sum_{k\geq 1} k^{-u}$. Let $X_1,X_2,\ldots$ be independent random variables, identically distributed, such that, for all $v>0$, $E(e^{vX_1})\leq e^{v^2\sigma^2/2}$. For every positive integer $t$, let $S_t=X_1+\ldots+X_t$. Then, for all $\gamma>1$ and $\displaystyle r\geq \frac{8}{(e-1)^2}$:

$$P\left(\bigcup_{t\in \mathbb{N}^*}\left\lbrace\vert S_t\vert>\sqrt{2\sigma^2t(r+\gamma\log\log(et))}\right\rbrace\right)\leq \sqrt{e}\zeta(\gamma(1-\frac{1}{2r}))(\frac{\sqrt{r}}{2\sqrt{2}}+1)^{\gamma}\exp(-r).$$
\end{lemma}

\begin{lemma}\label{lemma:theorem guarantee}~\\
Let $m\geq 1$ and $u\geq 20$. Then we have: 

$$m\geq 2u\log(\log(u))\Longrightarrow m\geq u\log(\log(m)).$$ 
\end{lemma}
\begin{proof}~\\
Define $\phi (m)=m-u\log(\log(m))$, and let $m_0=2u\log(\log(u))$. We have: 
\begin{align*}
\phi(m_0) &=2u\log(\log(u))-u(\log(\log(2u\log(\log(u)))))\\
          &= 2u\log(\log(u))-u\log(\log(2u)+\log(\log(\log(u))))
\end{align*}
It can be shown numerically that $\phi(m_0)\geq 0$ for $u\geq 20$.\\
Also, we have: $\phi'(m)=\frac{m\log(m)-u}{m\log(m)}\geq 0$ for all $m\geq m_0$ (notice that $m_0\geq u$ for $u\geq 20$). Then it is easy to see that $\phi(m)\geq \phi(m_0)$ for all $m\geq m_0$. This establishes the lemma.
\end{proof}

\begin{theorem}\label{theorem:savings label}~\\
Let $\delta$ $\in$ $(0,1)$, and $\epsilon$ $\in$ $(0,1)$. Let us assume that $w$ satisfies \eqref{guarantee-on-pool-0}. For $X_s$, set $\tilde{k}(\epsilon,\delta_s)$ $($with $\delta_s=\frac{\delta}{32s^2})$ as $$\tilde{k}(\epsilon,\delta_s)=\frac{c}{4\vert\eta(X_s)-\frac 12\vert^2}\left[\log(\frac{32s^2}{\delta})+\log\log(\frac{32s^2}{\delta})+\log\log\left(\frac{256\sqrt{e}}{\vert\eta(X_s)-\frac 12\vert}\right)\right],$$ where $c\geq 7.10^6$. For $k\geq 1$, $s\leq w$, let 
 $\Delta=\max(\frac{\epsilon}{2}, \left(\frac{\epsilon}{2C}\right)^{\frac{1}{\beta+1}})$ and $b_{\delta_s,k}$ defined in \eqref{eq:pi_ks}.\\
%Let us suppose the budget $n$ is sufficiently  large with respect to $k(\epsilon,\delta_s)$, for all $s\leq w$ such that the subroutine \texttt{ConfidentLabel} is independent of label budget $n$.  
Then, there exists an event $A_2$, such that $P(A_2)\geq 1-\delta/8$, and on $A_1\cap A_2$, we have:\\   
\begin{enumerate}
\item For $k\geq 1$, $\widehat{\eta}_k(X_s)$ and $\bar\eta_k(X_s)$ defined in  \eqref{eq:eta_Xs}, for all $s$ $\in$ $\lbrace 1,\ldots,w \rbrace$,

%and   $$\pi_{\delta}(k,s)=\sqrt{\frac 2k\left(\log\left(\frac{4c_1s^2}{\delta}\right)+ \log\log\left(\frac{4c_1s^2}{\delta}\right)+ \log\log(ek)\right)}$$

\begin{equation}
\vert \widehat \eta_k(X_s)-\bar\eta_k(X_s)\vert\ \leq b_{\delta_s,k}.
\label{eq:error-regresion}
\end{equation}
\item  For all $s\leq w$,  if $\vert\eta (X_s)-\frac 12\vert\geq \frac 12 \Delta$, then, $\tilde{k}(\epsilon,\delta_s)\leq k(\epsilon,\delta_s)$, and the subroutine \\ \texttt{ConfidentLabel}$(X_s)$:=\texttt{ConfidentLabel}$( X_s, k(\epsilon,\delta_s), t=\infty, \delta_s)$ uses at most $\tilde{k}(\epsilon,\delta_s)$  label requests. We also have 
\begin{equation}
\vert \frac{1}{\bar{k}_s}\sum_{i=1}^{\bar{k}_s}Y_{s}^{(i)}-\frac 12\vert\geq 2 b_{\delta_s,\bar k_s}
\label{guarantee-confident}
\end{equation}
and
\begin{equation}
f^*(X_s)= \mathds{1}_{\widehat{\eta}_{\bar{k}_s}(X_s)\geq \frac{1}{2}},
\label{guarantee-confident-2}
\end{equation}
Where  $\bar{k}_s$ is the number of requests made in \texttt{ConfidentLabel}$(X_s)$.

%where $$\widehat{\eta}_{\bar{k}_s}(X_s)=\frac{1}{\bar{k}_s}\sum_{i=1}^{\bar{k}_s}Y_{s}^{(i)}.$$
\end{enumerate}
\end{theorem}

\begin{proof}~\\
\begin{enumerate}
\item Let us begin with the proof of the first part of Theorem~\ref{theorem:savings label}.\\ Here, we follow the proof of Theorem 8 in \citep{kaufmann2016complexity}, with few additional modifications.\\ Let $s$ $\in$ $\lbrace 1,\ldots,w\rbrace.$ Set $\displaystyle S_k=\sum_{i=1}^{k}\left(Y_s^{(i)}-\eta(X_s^{(i)})\right)$. Given $\lbrace X_1,\ldots,X_w\rbrace$, $E(Y_s^{(k)}-\eta(X_s^{(k)}))=0$, and the random variables $\displaystyle\left\lbrace Y_s^{(i)}-\eta(X_s^{(i)}),\;i=1,\ldots,k\right\rbrace$ are independent. Then by Lemma~\ref{lemma:subgaussian}, given $\lbrace X_1,\ldots,X_w\rbrace$, as $Y_s^{(1)}-\eta(X_s^{(1)})$ takes values in $[-1,1]$,  we have $E(e^{v(Y_s^{(1)}-\eta(X_s^{(1)}))})\leq e^{v^2/2}$ for all $v>0$.  Furthermore, set $z=\log(\frac{32s^2}{\delta})$, and $r=z+3\log(z)$. We have $r\geq \frac{8}{(e-1)^2}$, and by Lemma~\ref{lemma:Kaufmann16}, with $\gamma=3/2$, we have: 
\begin{align*}
P&\left(\bigcup_{k\in \mathbb{N}^*}\left\lbrace\vert S_k\vert>\sqrt{2k(r+\gamma\log\log(ek))}\right\rbrace\right) \leq \sqrt{e}\zeta(3/2(1-\frac{1}{2r}))(\frac{\sqrt{r}}{2\sqrt{2}}+1)^{3/2}\exp(-r)\\
                                                                                                                 &= \frac{\sqrt{e}}{8}\zeta\left(\frac 32-\frac{3}{4(z+3\log(z))}\right)\frac{(\sqrt{z+3\log(z)}+\sqrt{8})^{3/2}}{z^3}\frac{\delta}{32s^2}
\end{align*}
It can be shown numerically that for $z\geq 2.03$, which holds for all $\delta$ $\in$ $(0,1)$, $s\geq 1$, 
$$\frac{\sqrt{e}}{8}\zeta\left(\frac 32-\frac{3}{4(z+3\log(z))}\right)\frac{(\sqrt{z+3\log(z)}+\sqrt{8})^{3/2}}{z^3}\leq 1.$$
Then, we have, given $s$ $\in$ $\lbrace 1,\ldots, w\rbrace$, there exists an event $A'_{2,s}$ such that $P(A'_{2,s})\geq 1- \delta/32s^2$, and simultaneously for all $k\geq 1$, we have: $$\vert S_k\vert\leq \sqrt{2k\left(\log\left(\frac{32s^2}{\delta}\right)+ \log\log\left(\frac{32s^2}{\delta}\right)+ \log\log(ek)\right)}.$$
By setting $\displaystyle A'_2=\cap_{s\geq 1}A'_{2,s}$, we have $P(A'_2)\geq 1-\delta/16$, and on $A'_2$, we have for all $s$ $\in$ $\lbrace 1,\ldots, w\rbrace$,  for all $k\geq 1$, 
$$\vert \widehat \eta_k(X_s)-\bar\eta_k(X_s)\vert\ \leq b_{\delta_s,k}.$$

\item  For the proof of the second part of Theorem~\ref{theorem:savings label}, we are going to show that there exists an event $A''_2$ such that \eqref{guarantee-confident} and \eqref{guarantee-confident-2} hold on $A'_2\cap A''_2\cap A_1$.\\ 
Given $\lbrace X_1,\ldots, X_w\rbrace$, and  $X_s$ $\in$ $\lbrace X_1,\ldots, X_w\rbrace$, by Lemma~\ref{lemma:subgaussian}, there exists an event $A''_{2,s}$, with $P(A''_{2,s})\geq 1-\delta/32s^2$, and on $A''_{2,s}$, we have: 
$$ \vert\widehat \eta_k(X_s)-\bar\eta_k(X_s)\vert\leq \sqrt{\frac{2\log(\frac{32s^2}{\delta})}{k}}.$$
This implies that: 
\begin{equation}\label{equation-hoeffding2}
 \vert\widehat{\eta}_k(X_s)-\frac 12\vert \geq \vert\bar\eta_k(X_s)-\frac 12\vert- \sqrt{\frac{2\log(\frac{32s^2}{\delta})}{k}}.
\end{equation}
On the event $A_1$, we have, for all $k\leq k_s$, by the $\alpha$-smoothness assumption \eqref{def:smooth}, 
\begin{equation}\label{proof-sufficient k0}
\vert\eta(X_s)-\eta(X_s^{(k)})\vert\leq \frac{31}{1024}\Delta.
\end{equation}
And then, if $\vert \eta(X_s)-\frac 12\vert\geq \frac 12 \Delta$, then $\vert \eta(X_s)-\frac 12\vert\geq \frac{1}{32} \Delta$  . The relation \eqref{proof-sufficient k0} becomes 
$$\vert \eta(X_s^{(k)})-\frac 12\vert\geq \frac{1}{1024}\vert\eta(X_s)-\frac 12\vert.$$

Then \eqref{equation-hoeffding2} becomes: 
\begin{equation}\label{equation-hoeffding3}
 \vert\widehat{\eta}_k(X_s)-\frac 12\vert \geq \frac{1}{1024}\vert\eta(X_s)-\frac 12\vert- \sqrt{\frac{2\log(\frac{32s^2}{\delta})}{k}}. 
\end{equation}
A sufficient condition for $k$ to satisfy \eqref{guarantee-confident}, is 
$$\frac{1}{1024}\vert\eta(X_s)-\frac 12\vert- \sqrt{\frac{2\log(\frac{32s^2}{\delta})}{k}}\geq 2 b_{\delta_s,k}$$
and then: 
$$\frac{1}{1024}\vert\eta(X_s)-\frac 12\vert- \sqrt{\frac{2\log(\frac{32s^2}{\delta})}{k}}\geq 2\sqrt{\frac 2k\left(\log\left(\frac{32s^2}{\delta}\right)+ \log\log\left(\frac{32s^2}{\delta}\right)+ \log\log(ek)\right)}$$
this implies: 
\begin{equation}\label{proof-sufficient k1}
k\geq \frac{1024}{\vert \eta(X_s)-\frac 12\vert^2}\left(\sqrt{2\log(\frac{32s^2}{\delta})}+ 2\sqrt{2\left(\log\left(\frac{32s^2}{\delta}\right)+ \log\log\left(\frac{32s^2}{\delta}\right)+ \log\log(ek)\right)}\right)^2.
\end{equation}

On the other hand, the right-hand side is  smaller than:

$$\frac{1024}{\vert \eta(X_s)-\frac 12\vert^2} \left(\sqrt{2\log(\frac{32s^2}{\delta})}+ 2\sqrt{2\log\left(\frac{32s^2}{\delta}\right)}+ 2\sqrt{2\log\log\left(\frac{32s^2}{\delta}\right)}+ 2\sqrt{2\log\log(ek)}\right)^2.$$
To deduce \eqref{proof-sufficient k1},  it suffices to have the expression into brackets lower than:

$$\frac{\sqrt{k}}{32}\vert\eta(X_s)-\frac 12\vert.$$
Then, it suffices to have simultaneously: 
\begin{equation*}
\sqrt{2\log(\frac{32s^2}{\delta})}\leq \frac{1}{9} \frac{\sqrt{k}}{32} \vert \eta(X_s)-\frac 12\vert
\end{equation*}

\begin{equation*}
\sqrt{2\log\log(\frac{32s^2}{\delta})}\leq \frac{1}{6} \frac{\sqrt{k}}{32}\vert \eta(X_s)-\frac 12\vert
\end{equation*}
\begin{equation*}
\sqrt{2\log\log(ek)}\leq  \frac{1}{6}\frac{\sqrt{k}}{32}\vert \eta(X_s)-\frac 12\vert
\end{equation*}

%\begin{equation*}
%\frac{c_0}{1+c_5}+\frac{c_0}{c_5}+\frac{c_0}{c_5}\leq 1
%% la somme des coefficients doit etre plus petite que 1 pour avoir <k 
%\end{equation*} 
Equivalently, we have:  
\begin{equation}\label{proof-guarantee-eq1}
k\geq \frac{1024}{\vert \eta(X_s)-\frac 12\vert^2}\, 162\log(\frac{32s^2}{\delta})
\end{equation}

\begin{equation}\label{proof-guarantee-eq2}
k\geq \frac{1024}{\vert \eta(X_s)-\frac 12\vert^2}\,72\log\log(\frac{32s^2}{\delta})
\end{equation}

\begin{equation}\label{proof-guarantee-eq3}
k\geq \frac{1024}{\vert \eta(X_s)-\frac 12\vert^2}\,72\log\log(ek)
\end{equation}

%\begin{equation}\label{proof-guarantee-eq4}
%\frac{c_0}{1+c_5}+\frac{c_0}{c_5}+\frac{c_0}{c_5}\leq 1
%\end{equation}

We can apply the Lemma~\ref{lemma:theorem guarantee} in \eqref{proof-guarantee-eq3} by taking: $m=ek$ and $u=\frac{73728e}{\vert \eta(X_s)-\frac 12\vert^2}$. We have $m\geq 1$ and $u\geq 20$ and then, a sufficient condition to have \eqref{proof-guarantee-eq3} is: 

$$k\geq 2\frac{73728e}{\vert\eta(X_s)-\frac 12\vert^2}\log\log\left(\frac{73728e}{\vert\eta(X_s)-\frac 12\vert^2}\right)$$ 
or
\begin{equation}\label{proof-guarantee-eq5}
k\geq 4\frac{73728e}{\vert\eta(X_s)-\frac 12\vert^2}\log\log\left(\frac{\sqrt{73728e}}{\vert\eta(X_s)-\frac 12\vert}\right)
\end{equation}

We can easily see that $\tilde{k}_s:=\tilde{k}(\epsilon,\delta_s)$ satisfies \eqref{proof-guarantee-eq1}, \eqref{proof-guarantee-eq2}, \eqref{proof-guarantee-eq5}.  Then 
\begin{equation}\label{proof-guarantee-eq6}
\vert \frac{1}{\tilde{k}_s}\sum_{i=1}^{\tilde{k}_s}Y_{s}^{(i)}-\frac 12\vert\geq 2b_{\delta_s,\tilde{k}_s}.
\end{equation} 

As $\vert \eta(X_s)-\frac 12\vert\geq \frac 12 \Delta$, we can easily see that $\tilde{k}(\epsilon,\delta_s)\leq k(\epsilon,\delta_s)$. 
By taking the minimum value $\bar{k}_{s}=\bar k(\epsilon,\delta_s)$ that satisfies \eqref{proof-guarantee-eq6}, we can see that when the budget allows us, the subroutine \texttt{ConfidentLabel} requests  $\bar{k}_{s}$ labels,  and we have:
 \begin{equation}\label{proof-guarantee-eq7}
\vert \frac{1}{\bar{k}_s}\sum_{i=1}^{\bar{k}_s}Y_{s}^{i}-\frac 12\vert\geq 2 b_{\delta_s,\bar{k}_s}.
\end{equation}  
By setting $A''_2=\displaystyle\cap_{s\geq 1}A''_{2,s}$, we have $P(A''_2)\geq 1-\delta/16$, and we can deduce \eqref{guarantee-confident}.
\\~\\
We have on $A'_2$,  for all $s\leq w$,  $k\leq k(\epsilon,\delta_s)$, 
$$\vert\widehat{\eta}(X_s)-\bar{\eta}_k(X_s)\vert\leq b_{\delta_s,k}.$$
And then, on $A_1\cap A'_2$, we have for all $s\leq w$,  $k\leq k(\epsilon,\delta_s)$ : 

\begin{align}
\vert \eta(X_s)-\widehat{\eta}_k(X_s)\vert &\leq \vert \eta(X_s)-\bar{\eta}_k(X_s)\vert+ \vert \bar{\eta}(X_s)-\widehat{\eta}_k(X_s)\vert\nonumber\\
                                           & \leq \frac{31}{1024}\Delta+ b_{\delta_s,k}.\label{eq00}                                           
\end{align}
%Moreover, we have: $$\vert \eta(X_s)-\frac 12\vert\leq \vert\widehat{\eta}(X_s)-\frac 12\vert+ \vert\
Assume without loss of generality that $\eta(X_s)\geq \frac 12$, which leads to: 

\begin{align}
\widehat{\eta}_{\bar{k}_s}(X_s)-\frac 12 &=\widehat{\eta}_{\bar{k}_s}(X_s)-\eta(X_s)+\eta(X_s)-\frac 12\nonumber\\
                                    &\geq -\vert\widehat{\eta}_{\bar{k}_s}(X_s)-\eta(X_s)\vert+ \eta(X_s)-\frac 12.\label{eq000}
\end{align}

If $\eta(X_s)-\frac 12\geq \frac 12\Delta$,  with \eqref{eq00}, the expression \eqref{eq000} becomes: 

\begin{align}
\widehat{\eta}_{\bar{k}_s}(X_s)-\frac 12 &\geq -\frac{31}{1024}\Delta-b_{\delta_s,\bar{k}_s}+\frac 12\Delta\nonumber\\
                                    &= \frac{481}{1024}\Delta-b_{\delta_s,\bar{k}_s}\nonumber\\
                                    &\geq -b_{\delta_s,\bar{k}_s} \label{eq0000}
\end{align}

On the other hand, we have by \eqref{guarantee-confident}, 
$$\vert\widehat{\eta}_{\bar{k}_s}(X_s)-\frac 12\vert\geq 2b_{\delta_s,\bar{k}_s},$$ that is to say: 

$$\widehat{\eta}_{\bar{k}_s}(X_s)-\frac 12\geq 2b_{\delta_s,\bar{k}_s} \quad \text{or}\quad \widehat{\eta}_{\bar{k}_s}(X_s)-\frac 12\leq -2b_{\delta_s,\bar{k}_s}.$$
By \eqref{eq0000}, we have  necessarily $\widehat{\eta}_{\bar{k}_s}(X_s)-\frac 12\geq 2b_{\delta_s,\bar{k}_s}$, and then:
$$\widehat{\eta}_{\bar{k}_s}-\frac 12\geq \max(-b_{\delta_s,\bar{k}_s},2b_{\delta_s,\bar{k}_s})=2b_{\delta_s,\bar{k}_s}\geq 0,$$
Thus we can easily deduce \eqref{guarantee-confident-2}.\\
By setting $A_2=A'_2\cap A''_2$, we have $P(A_2)\geq 1-\delta/8$ and on $A_1\cap A_2$, the item 1 and item 2 hold simultaneously.
\end{enumerate}
\end{proof}
\subsection{Sufficient condition to be an informative point}
\label{subsec: suff-cond-informativ}
As noticed in Section~\ref{sec:Reliable}, a sufficient condition for a point $X_t$ (with $t\leq w$) to be considered as \textit{not} informative is: 
\begin{equation}\label{equation:suffi_uninformativ}
\min(P_X(B(X_t,\rho(X_t,X_s))),P_X(B(X_s,\rho(X_t,X_s))))\leq O((\widehat{LB}_s)^{d/\alpha}).
\end{equation} for some previous informative point $X_s$ (with $(X_s,\widehat{Y}_s, \widehat{LB}_{s})\in \widehat{S}$ the \textit{current active set} just before attaining $X_t$ in \texttt{KALLS}(Algorithm\eqref{algo:KALLS})). Because $P_X$ is unknown, we provide a computational scheme sufficient to obtain ~\eqref{equation:suffi_uninformativ}.\\Firstly we follow the general procedure used in \citep{kontorovich2016active} to estimate adaptively the expectation of a Bernoulli random variable. And secondly, we apply it to the Bernoulli variable $\mathds{1}_{A}$ where $A=\lbrace x,\;x\in B(X,r)\rbrace$ for $r>0$ and $X$ $\in$ $\mathcal{X}$.
\begin{lemma}\citep{kontorovich2016active}\label{lemma-kontorovich}~\\
Let $\delta'$ $\in$ $(0,1)$, $\epsilon_o>0$, $t\geq 7$ and set $g(t)=1+\frac{8}{3t}+\sqrt{\frac{2}{t}}$. Let $p_1,p_2,\ldots$ $\in$ $\lbrace0,1\rbrace$ be i.i.d Bernoulli random variables with expectation $p$. Let $\widehat{p}$ be the output of \texttt{BerEst}$(\epsilon_o,\delta',t)$. There exists an event $A'$, such that  $P(A')\geq 1-\delta'$, and on $A'$, we have:
\begin{enumerate}
\item If $\widehat{p}\leq \frac{\epsilon_o}{g(t)}$ then $p\leq \epsilon_o$, otherwise, we have $p\geq \frac{2-g(t)}{g(t)}\epsilon_0$.
\item The number of random draws in the \texttt{BerEst} subroutine (Algorithm~\ref{algo:BerEst}) is at most $\frac{8t\log(\frac{8t}{\delta'\psi})}{\psi}$, where $\psi :=\max(\epsilon_o, \frac{p}{g(t)})$. 
\end{enumerate} 
\end{lemma}

\begin{lemma}\label{lemma:sufficient-informative}~\\
Let $\epsilon$, $\delta$ $\in$ $(0,1)$, $r>0$. 
%For $(X_{s'},\widehat{Y}_{s'},\widehat{LB}_{s'})$ $\in$ $\wid$, (where $I$ is the \textit{rough active set} defined in \texttt{KALLS}),  let $$\widehat{LB}_s=\left|\frac{1}{\vert Q_s\vert}\sum_{(X,Y)\in Q_s} Y-\frac 12\right|-b_{\delta_s,\vert Q_s\vert}$$ and $ Q_s$ is defined in subroutine \texttt{ConfidentLabel} (Algorithm \ref{algo:ConfidentLabel}). 
Let us assume that $w$ satisfies \eqref{condition2-0}.
%\begin{equation}\label{condition2}
%w\geq \frac{400\log\left(\frac{12800w^2}{\delta (c_3\bar c\sigma)^{d/\alpha}}\right)}{(c_3\bar c\sigma)^{d/\alpha}}
%\end{equation}
%il faut prendre n tres petit devant w
%where $$\sigma=\sqrt{\frac {1}{n}\left(\log\left(\frac{4c_1}{\delta}\right)+ \log\log\left(\frac{4c_1}{\delta}\right)\right)},\;\;\;n\;\text{the label budget}$$ and $c_1$, $c_2$, $\bar c$ $c_3$  are defined in Section \ref{subsec:constants}\\

There exists an event $A_3$, such that $P(A_3)\geq 1-\delta/16$, we have, on $A_3$, for all $s\leq w$: \\
If there exists $1\leq s'< s$, such that $X_{s'}$ is an informative point, and $(X_{s'},\widehat{Y}_{s'},\widehat{LB}_{s'})$ $\in$ $\widehat{S}$ (the current active set just before attaining $X_s$ defined in \texttt{KALLS}$($Algorithm\eqref{algo:KALLS}$))$, and that satisfies: 
\begin{equation}
\left(\widehat{p}_{X_{s'}}\leq \frac{75}{94}\left(\frac{1}{64L}\widehat{LB}_{s'}\right)^{d/\alpha}\,or\,\widehat{p}_{X_{s}}\leq \frac{75}{94}\left(\frac{1}{64L}\widehat{LB}_{s'}\right)^{d/\alpha}\right)
\label{equation-corolary-konto0}
\end{equation}
 where $$\widehat{p}_{X_{s'}}:=\texttt{Estprob}(X_{s'},\rho(X_s,X_{s'}), \left(\frac{1}{64L}\widehat{LB}_{s'}\right)^{d/\alpha},50,\delta_s)$$ and 
$$\widehat{p}_{X_{s}}:=\texttt{Estprob}(X_s, \rho(X_s,X_{s'}), \left(\frac{1}{64L}\widehat{LB}_{s'}\right)^{d/\alpha},50,\delta_s)$$ then
\begin{equation}\label{equation-corolary-konto01}
\min(P_X(B(X_{s},\rho(X_{s'},X_s))),P_X(B(X_{s'},\rho(X_{s'},X_s))))\leq \left(\frac{1}{64L}\widehat{LB}_{s'}\right)^{d/\alpha}.
\end{equation}
Otherwise, if \eqref{equation-corolary-konto0} does not holds,  i.e:  $$\min(\widehat{p}_{X_{s'}},\widehat{p}_{X_{s}})> \frac{75}{94}\left(\frac{1}{64L}\widehat{LB}_{s'}\right)^{d/\alpha},$$ then 
\begin{equation}\label{equation-corolary-konto1}
\min(P_X(B(X_{s},\rho(X_{s'},X_s))),P_X(B(X_{s'},\rho(X_{s'},X_s))))\geq \frac{28}{47}\left(\frac{1}{64L}\widehat{LB}_{s'}\right)^{d/\alpha}.
\end{equation}
\end{lemma}

\begin{proof}~\\
By following the scheme of subroutine \texttt{Estprob}, this Lemma is a direct application of Lemma \ref{lemma-kontorovich} by taking for all $s\leq w$, $t=50$, $\epsilon_o=\left(\frac{1}{64L}\widehat{LB}_{s'}\right)^{d/\alpha}$, $\delta'=\delta_s$, $r=\rho(X_s,X_{s'})$, $A_{3,s}:=A'$. And then, if we set $A_3=\cap_{s\geq 1}A_{3,s}$, we have $P(A_3)\geq 1-\delta/16$, and on the event $A_3$, we can easily deduce \eqref{equation-corolary-konto01} and \eqref{equation-corolary-konto1} in each cases. 
\\
On the other hand, for all $s\leq w$, the number of draws in \texttt{Estprob}$(X_s,\rho(X_s,X_{s'}), \left(\frac{1}{64L}\widehat{LB}_{s'}\right)^{d/\alpha},50,\delta_s)$ (respectively \texttt{Estprob}$(X_{s'},\rho(X_s,X_{s'}), \left(\frac{1}{64L}\widehat{LB}_{s'}\right)^{d/\alpha},50,\delta_s)$) is always lower than $w$.  Indeed, by Lemma \ref{lemma-kontorovich}, the number of draws is at most:
$$N:=\frac{400\log(\frac{12800s^2}{\delta\psi})}{\psi}\quad\text{where}\quad \psi=\max((\frac{1}{64L}\widehat{LB}_{s'})^{d/\alpha}, \frac{75}{94}P_X(B(X_s,\rho(X_s,X_{s'})))).$$ 
Then we have: 
\begin{align}
N &\leq \frac{400\log\left(\frac{12800s^2}{\delta (\frac{1}{64L}\widehat{LB}_{s'})^{d/\alpha}}\right)}{(\frac{1}{64L}\widehat{LB}_{s'})^{\frac{d}{\alpha}}}\nonumber\\
  &\leq \frac{400\log\left(\frac{12800s^2}{\delta (\frac{1}{64L}\bar{c}b_{\delta_{s'},\vert Q_{s'}\vert})^{d/\alpha}}\right)}{(\frac{1}{64L}\bar{c}b_{\delta_{s'},\vert Q_{s'}\vert})^{d/\alpha}}\;\;\quad(\text{as}\; \widehat{LB}_{s'}\geq \bar c b_{\delta_{s'},\vert Q_{s'}\vert},\;\text{with}\,\bar c=0.1)\label{eq:labelestprob}\\
  & \leq \frac{400\log\left(\frac{12800w^2}{\delta (\frac{1}{64L}\bar c\phi_n)^{d/\alpha}}\right)}{(\frac{1}{64L}\bar c\phi_n)^{d/\alpha}}\;\;\quad (\text{we can easily see that}\; b_{\delta_{s'},\vert Q_{s'}\vert}\geq \phi_n)\nonumber\\
  &\leq w \;\;\quad (\text{by}~\eqref{condition2-0})\nonumber.
\end{align}
In equation \eqref{eq:labelestprob}, $b_{\delta_{s'},\vert Q_{s'}\vert}$ is defined by \eqref{eq:pi_ks}, and $\vert Q_{s'}\vert$ represents the number of label requests used in the subroutine \texttt{ConfidentLabel}$($Algorithm \eqref{algo:ConfidentLabel}$)$ at the stage $s'$.  
%A bit analysis shows that ~\eqref{condition2} is satisfied when $w$ is very large with respect to $n$. 
\end{proof}

\subsection{Label the instance space}
\label{subsec:lab-instance-space}

\begin{theorem}~\\
\label{theorem:lab-instance-space}
Let $\epsilon$, $\delta$ $\in$ $(0,1)$. Let 
\begin{equation}
T_{\epsilon,\delta}=\frac{1}{\tilde p_{\epsilon}}\log(\frac{8}{\delta}), \;\text{and}\; \tilde p_{\epsilon}=\left(\frac{\Delta}{128L}\right)^{d/\alpha},\; \text{with}\; \Delta=\max(\frac{\epsilon}{2}, \left(\frac{\epsilon}{2C}\right)^{\frac{1}{\beta+1}})
\label{eq:label-instance}
\end{equation} 

Let $I$ the set of indexes of informative points used in \texttt{KALLS} (Algorithm \ref{algo:KALLS}). Let us consider its last update in \texttt{KALLS} (Algorithm \ref{algo:KALLS}) and also denoted it by $I$.\\  Then, set $\displaystyle s_{I}=\max I $  the index of the last informative point. Let $\widehat{S}_{ac}$ be the \textit{active set} obtained in \texttt{KALLS} (Algorithm \ref{algo:KALLS}) and denote by $\widehat{f}_{n,w}$ the output  \texttt{1NN}$(\widehat{S}_{ac})$. There exists an event $A_4$ such that $P(A_4)\geq 1-\delta/8$, and on $A_1\cap A_2\cap A_3 \cap A_4$, we have 
\begin{enumerate}
\item 
\begin{equation}\label{lemma:covering space}
\sup_{x\in supp(P_X)}\,\min_{\bar X\in \lbrace X_1,\ldots,X_{T_{\epsilon,\delta}}\rbrace} P_X(B(x,\rho(\bar X,x)))\leq \tilde{p}_{\epsilon}. 
\end{equation}
\item If $w$ satisfies \eqref{guarantee-on-pool-0} and \eqref{condition2-0}  and the following condition holds\begin{equation}\label{necessary-condition-budget}
s_{I}\geq T_{\epsilon, \delta},
\end{equation}
then, for all $x$ $\in$ supp$(P_X)$ such that $\vert\eta(x)-\frac 12\vert > \Delta$, there exists $s:=s(x)$ $\in$ $I$ such that: 

\begin{equation}
\vert\eta(X_s)-\frac 12\vert\geq \frac 12 \Delta
\label{eq:related-informative}
\end{equation}
and 

\begin{equation}
f^*(x)=f^*(X_s).
\label{eq:right-label}
\end{equation} 
In addition, we have
\begin{equation}
\widehat{f}_{n,w}(x)= f^*(x).
\label{eq:label-1-nn}
\end{equation}
\end{enumerate}
\end{theorem}

\begin{proof}~\\
This proof is based on results from \citep{hannekenonparametric} with few additional modifications. 
\begin{enumerate}
\item Let us begin by proving the first part of Theorem~\ref{theorem:lab-instance-space}.\\
For $x$ $\in$ supp($P_X$), let us introduce $$r_{\tilde p_{\epsilon}}(x)=\inf\lbrace r>0,\;P_X(B(x,r))\geq \tilde p_{\epsilon}\rbrace.$$
By Lemma \ref{lemme-chaudhuri_lemma}, we have $P_X(B(x,r_{\tilde p_{\epsilon}}(x))\geq \tilde p_{\epsilon}$. Then each $\bar X$ $\in$ $\lbrace X_1,\ldots, X_{T_{\epsilon,\delta}}\rbrace$  belongs to $B(x,r_{\tilde p_{\epsilon}}(x))$ with probability at least $\tilde p_{\epsilon}$. If we denote $\widehat{P}$ the probability over the data, we have: 
\begin{align*}
&\widehat{P}(\exists \bar X \in \lbrace X_1,\ldots, X_{T_{\epsilon,\delta}}\rbrace,\, P_X(B(x,\rho(x,\bar X))\leq \tilde p_{\epsilon})\\ &= 1-\widehat{P}(\forall \bar X\in \lbrace X_1,\ldots, X_{T_{\epsilon,\delta}}\rbrace, \, P_X(B(x,\rho(x,\bar X))> \tilde p_{\epsilon})\\
                                                                                                       &= 1-\prod_{i=1}^{T_{\epsilon,\delta}}\widehat{P}(P_X(B(x,\rho(x,X_i))>\tilde p_{\epsilon})\\
                                                                                                       &\geq 1-\prod_{i=1}^{T_{\epsilon,\delta}}\widehat{P}(\rho(x,X_i)> r_{\tilde p_{\epsilon}}(x))\\
                                                                                                       &= 1-\prod_{i=1}^{T_{\epsilon,\delta}}(1-\widehat{P}(\rho(x,X_i)\leq r_{\tilde p_{\epsilon}}(x)))\\
                                                                                                       &\geq 1-(1-\tilde p_{\epsilon})^{T_{\epsilon,\delta}}\\
                                                                                                       &\geq 1-\exp(-T_{\epsilon,\delta}\tilde p_{\epsilon})\\
                                                                                                       &=1-\delta/8.
\end{align*}
Then, there exists an event $A_4$, such that $P(A_4)\geq 1-\delta/8$ and \eqref{lemma:covering space} holds on $A_4$. 
And then, we can easily conclude the first part.

\item For the second part of Theorem~\ref{theorem:lab-instance-space}, let $x$ $\in$ supp$(P_X)$. By \eqref{lemma:covering space}, on $A_4$ there exists $X_x$ $\in$ $\lbrace X_1,\ldots,X_{T_{\epsilon,\delta}}\rbrace$ such that: 

\begin{equation}\label{eq:l lab}
P_X(B(x,\rho( X_x,x)))\leq \tilde{p}_{\epsilon}.
\end{equation}
By assumption \eqref{def:smooth}, we have: 
\begin{equation}\label{eq:2 lab}
\vert\eta(x)-\eta(X_x)\vert\leq \frac{1}{128}\Delta<\frac{1}{32}\Delta.
\end{equation}

Then if $\vert\eta(x)-\frac 12\vert> \Delta$, we have:                                                                  
\begin{equation}\label{eq:22 lab}
(1-\frac{1}{32})\Delta<\vert\eta(X_x)-\frac 12\vert< (1+\frac{1}{32})\Delta.
\end{equation}
%Let $I$ the set of index of informative points used in \texttt{KALLS}\ref{algo:KALLS}, and $s_I=\max\lbrace s,\;\;s\in I\rbrace$. 
As $s_I\geq T_{\epsilon,\delta}$, then there exists $s'$ such that  $X_x:=X_{s'}$ and $X_{s'}$ passes through the subroutine \texttt{Reliable}.

We have two cases:
\begin{enumerate}
\item[a)] \textbf{$X_{s'}$ is uninformative}. Then there exists $s<s'$, such that $X_s$ is an informative point, and  $$\widehat{LB}_{s}\geq 0.1b_{\delta_s,\vert Q_{s}\vert}\;\; \text{and}\;\; \min(\widehat{p}_{X_{s}},\widehat{p}_{X_{s'}})\leq \frac{75}{94}\left(\frac{1}{64L}\widehat{LB}_s\right)^{d/\alpha}$$ where $\widehat{p}_{X_{s}}:=\texttt{Estprob}(X_s,\rho(X_s,X_{s'}), \left(\frac{1}{64L}\widehat{LB}_{s}\right)^{d/\alpha},50,\delta_s$ $)$, and\\ $\widehat{p}_{X_{s'}}:=\texttt{Estprob}(X_{s'},\rho(X_s,X_{s'}), \left(\frac{1}{64L}\widehat{LB}_{s}\right)^{d/\alpha},50,\delta_s)$   then by Lemma \ref{lemma:sufficient-informative}, 
\begin{equation}\label{equation-corolary-konto}
\min(P_X(B(X_{s},\rho(X_s,X_{s'}))),P_X(B(X_{s'},\rho(X_s,X_{s'}))))\leq \left(\frac{1}{64L}\widehat{LB}_{s}\right)^{d/\alpha}.
\end{equation} 
Necessary, we have $\vert \eta(X_s)-\frac 12\vert\geq \frac{1}{32}\Delta$. Indeed,  if $\vert \eta(X_s)-\frac 12\vert< \frac{1}{32}\Delta$, then on $A_1\cap A_2$, by denoting $\bar k_s$  the number of request labels in $\texttt{ConfidentLabel}(X_s):=\texttt{ConfidentLabel}(X_s, k(\epsilon,\delta_s),t,\delta_s)$, $($where $\displaystyle t=n-\sum_{s_i\in I,s_i<s} \vert Q_{s_i}\vert$ and $\vert Q_{s_i}\vert$ the number of label requests used in \texttt{ConfidentLabel}$(X_{s_i})$ $)$ \\We have: 
\begin{align}
\widehat{LB}_{s} &= \vert \widehat{\eta}_{\bar k_s}(X_s)-\frac 12\vert-b_{\delta_s,\bar k_s}\nonumber \\
                    &\leq \vert \widehat{\eta}_{\bar k_s}(X_s)-\bar{\eta}_{\bar k_s}(X_s)\vert + \vert \bar{\eta}_{\bar k_s}(X_s)-\frac 12\vert- b_{\delta_s,\bar k_s} \nonumber\\ 
                    &\leq \vert \bar{\eta}_{\bar k_s}(X_s)-\frac 12\vert\quad (\text{by \eqref{eq:error-regresion}}) \nonumber\\ 
                    &\leq \vert \eta(X_s)-\frac 12\vert\nonumber\\ &+\frac{1}{32}(1-\frac{1}{32})\Delta \quad (\text{by assumption \eqref{def:smooth} and Theorem \ref{theo:passive-guarantee}}) \label{eq30: lab}\\
                    &< \frac{1}{32}\Delta+ \frac{1}{32}(1-\frac{1}{32})\Delta \nonumber\\
                    & =\frac{63}{1024}\Delta\label{eq3: lab}                    
\end{align}
 
By assumption \eqref{def:smooth} and \eqref{equation-corolary-konto}, we have: 
\begin{align*}
\vert\eta(X_{s'})-\frac 12\vert &\leq \vert\eta(X_{s})-\frac 12\vert +\frac{1}{64}\widehat{LB}_{s}\\
                                & < \frac{1}{32}\Delta+\frac{1}{64}.\frac{63}{1024}\Delta\quad (\text{by \eqref{eq3: lab}})\\
                                & = (\frac{1}{32}+\frac{1}{64}.\frac{63}{1024})\Delta\\   
                                &\leq (1-\frac{1}{32})\Delta        
\end{align*}
that contradicts \eqref{eq:22 lab}, then we have $\vert \eta(X_s)-\frac 12\vert\geq \frac{1}{32}\Delta$.
Therefore, by \eqref{equation-corolary-konto}, \eqref{eq30: lab}, we have: 

\begin{align}
 P_X(B(X_{s'},\rho(X_s,X_{s'}))) &\leq \left(\frac{1}{64L}\widehat{LB}_{s}\right)^{d/\alpha}\nonumber\\
                               &\leq \left(\frac{1}{64L}\left(\vert \eta(X_s)-\frac 12\vert +\frac{1}{32}(1-\frac{1}{32})\Delta\right)\right)^{d/\alpha}\nonumber\\
                               &\leq \left(\frac{1}{64L}\left(\vert \eta(X_s)-\frac 12\vert +(1-\frac{1}{32})\vert \eta(X_s)-\frac 12\vert\right)\right)^{d/\alpha}\nonumber\\
                               &= \left(\frac{1}{64L}(2-\frac{1}{32})\vert \eta(X_s)-\frac 12\vert\right)^{d/\alpha}\nonumber \\
                               &=\left(\frac{63}{2048L}\vert \eta(X_s)-\frac 12\vert\right)^{d/\alpha}. \label{eq4: lab}
\end{align}

On the other hand, by \eqref{eq:l lab}, we have: 
\begin{align}
 P_X(B(x,\rho( X_{s'},x))) &\leq \tilde{p}_{\epsilon}\nonumber\\
                          &= \left(\frac{1}{128L}\Delta \right)^{d/\alpha}\nonumber\\
                          &\leq \left(\frac{1}{128L}\vert \eta(x)-\frac 12\vert \right)^{d/\alpha}\label{eq5: lab}.         
\end{align}
We have: 
\begin{align}
\vert\eta(x)-\eta(X_s)\vert &\leq \vert\eta(x)-\eta(X_{s'})\vert+ \vert\eta(X_{s'})-\eta(X_s)\vert\nonumber\\
                            &\leq  L .P_X(B(x,\rho( X_{s'},x)))^{\alpha/d}+ L .P_X(B(X_{s'},\rho( X_{s'},X_s)))^{\alpha/d}\quad (\text{by assumption \eqref{def:smooth}})\nonumber\\
                            &\leq \frac{1}{128}\vert \eta(x)-\frac 12\vert + \frac{63}{2048}\vert \eta(X_s)-\frac 12\vert \quad (\text{by \eqref{eq4: lab} and \eqref{eq5: lab}})\label{eq600: lab}\\
                            & \leq \frac{1}{128}\vert \eta(x)-\frac 12\vert + \frac{\frac{63}{2048}}{1-\frac{63}{2048}}\vert \eta(X_{s'})-\frac 12\vert \quad (\text{by assumption \eqref{def:smooth} and \eqref{eq4: lab}})\nonumber\\
                            & \leq \frac{1}{128}\vert \eta(x)-\frac 12\vert + \frac{63}{1985}(1+\frac{1}{128})\vert \eta(x)-\frac 12\vert  \quad (\text{by \eqref{eq:2 lab}})\nonumber \\
                            &=\frac{79}{1985}\vert \eta(x)-\frac 12\vert\label{eq6:lab}
\end{align}
%j'aimerai plutot avoir $P_X(B(x,\rho( X_{s},x)))\leq P_X(B(x,\rho( X_{s'},x)))+P_X(B(X_{s'},\rho( X_{s'},X_s)))$
%\danger
\item[b)] \textbf{$X_{s'}$ is informative}. In this case, $s=s'$  and then we always obtains the equation \eqref{eq6:lab}, which becomes
\end{enumerate}

\begin{align}
\vert\eta(X_s)-\frac 12\vert &\geq \left(1-\frac{79}{1985}\right)\vert\eta(x)-\frac 12\vert\label{eq60: lab}\\
                                           &\geq \left(1-\frac{79}{1985}\right)\Delta \nonumber\\
                                           &\geq \frac 12 \Delta
\end{align}                                    
Then 
\begin{equation}\label{eq7: lab}
\vert\eta(X_s)-\frac 12\vert\geq \frac{1}{2}\Delta
\end{equation}
On $A_1\cap A_2$, by Theorem \ref{theorem:savings label}, the subroutine $\texttt{ConfidentLabel}(X_s)$ uses at most $\tilde{k}(\epsilon,\delta_s)$ request labels, and returns the correct label (with respect to the Bayes classifier) of $X_s$. \\
Let us proof that $f^*(x)=f^*(X_s)$. Let us assume without loss of generality that $\eta(X_s)-\frac 12\geq 0$. We will show that $\eta(x)-\frac 12\geq 0$. We have:   
\begin{align*}
\eta (x)-\frac 12 &=\eta(x)-\eta(X_s)+\eta(X_s)-\frac 12\\
                  &\geq \eta(X_s)-\frac 12-\frac{79}{1985}\vert \eta(x)-\frac 12\vert \quad (\text{by \eqref{eq6:lab}})\\
                  &\geq (1-\frac{79}{1985})\vert \eta(x)-\frac 12\vert\\&-\frac{79}{1985}\vert \eta(x)-\frac 12\vert \quad (\text{by \eqref{eq6:lab}})\\
                  &=\frac{1827}{1985}\vert \eta(x)-\frac 12\vert\\
%                  &\geq (1-c_3(2-c_0))\left(1-c_7-\frac{c_3(2-c_0)(1+c_7)}{1-c_3(2-c_0)}\right)\vert \eta(x)-\frac 12\vert-c_7\vert \eta(x)-\frac 12\vert\quad (\text{by \eqref{eq6:lab}})\\
%                  &=\left((1-c_3(2-c_0))\left(1-c_7-\frac{c_3(2-c_0)(1+c_7)}{1-c_3(2-c_0)}\right)-c_7\right)\vert \eta(x)-\frac %12\vert\\
                  &\geq 0 \quad 
%(\text{see Section \ref{subsec:constants}})
\end{align*} 
Then $f^*(x)=f^*(X_s)$.\\
As $\vert\eta(X_s)-\frac 12\vert\geq \frac{1}{2}\Delta$, by using Theorem \ref{theorem:savings label} (the second part), we can easily see that $(X_s,\widehat{Y}_s)$ $\in$ $\widehat{S}_{ac}$ $($where $\widehat{Y}_s$ is the inferred label of $X_s$ provided by the subroutine \texttt{ConfidentLabel} in \texttt{KALLS}(Algorithm \eqref{algo:KALLS})).  
\\ Let $X^{(1)}_x$ the nearest neighbor of $x$ in $\widehat{S}_{ac}$. We have: 
\begin{align}
\vert\eta(x)-\eta(X^{(1)}_x)\vert & \leq   L.P_X(B(x,\rho(x,X^{(1)}_x)))^{\alpha/d}\nonumber\\
                                &\leq L.P_X(B(x,\rho(x,X_s))))^{\alpha/d}\nonumber\\
                                &\leq L.P_X(B(x,\rho(x,X_{s'})))^{\alpha/d}\nonumber+L.P_X(B(X_{s'},\rho(X_{s'},X_s)))^{\alpha/d}\nonumber\\
                                & \leq \frac{79}{1985}\vert \eta(x)-\frac 12\vert \quad\text{by}\;\eqref{eq6:lab} \label{eq:label-bayes}
\end{align} 
Then, $\vert \eta(X^{(1)}_x)-\frac 12\vert\geq (1-\frac{79}{1985}) \vert \eta(x)-\frac 12\vert\geq \frac 12\Delta $ and by Theorem \ref{theorem:savings label}, the subroutine \texttt{ConfidentLabel}($X^{(1)}_x$) outputs 
\begin{equation}
\widehat{Y}^{(1)}_x=f^*(X^{(1)}_x).
\label{eq:label-confident1}
\end{equation}
 Furthermore, \eqref{eq:label-bayes} implies $$\vert\eta(x)-\eta(X^{(1)}_x)\vert \leq \vert \eta(x)-\frac 12\vert$$
then $f^*(x)=f^*(X^{(1)}_x)$. With \eqref{eq:label-confident1}, we easily deduce that: 
$$f_{n,w}(x)=\widehat{Y}^{(1)}_x=f^*(X^{(1)}_x)= f^*(x).$$
\end{enumerate}
 
\end{proof}
%\begin{theorem}~\\
%\label{lemma:agree-label}
%Let $x$ $\in$ supp$(P_X)$ such that $\vert\eta(x)-\frac 12\vert > \widehat{\Delta}$. Let $I\subset \mathcal{X}\times\mathbb{R}\times\mathbb{N}$ the set used in \texttt{KALLS}.\\  Set $\displaystyle s_{I}=\max I'$, with $I'=\lbrace s, (X_s, \widehat{LB}_s, \vert Q_s\vert) \in I\rbrace$ $($$s_I$ is the index of the last informative point$)$. Let us assume that $s_{I} \geq T_{\epsilon, \delta}$. Let $\widehat{\mathcal{S}}$ the final active set use in subroutine \texttt{Learn} \eqref{algo:learn}.  Let $\widehat{f}_{n,w}$ the output of the subroutine \texttt{Learn}. Let us assume that $w$ satisfies \eqref{guarantee-on-pool-0} and \eqref{condition2-0}. We have on $A_1\cap A_2\cap A_3\cap A_4$
%$$\widehat{f}_{n,w}(x)= f^*(x).$$
%\end{theorem}

\subsection{Label complexity}
\label{subsec:label complexity}
\begin{lemma}~\\
\label{lemma:label-complexity}
Let us assume that $w$ satisfies \eqref{guarantee-on-pool-0}, \eqref{condition2-0}, and $w\geq T_{\epsilon,\delta}$. Then, there exists an event $A_5$ such that $P(A_5)\geq 1-\delta/8$, and on $A_1\cap A_2\cap A_3\cap A_5$. The condition \eqref{eq:label-complexity} is sufficient to guarantee \eqref{necessary-condition-budget}.
\end{lemma}
Before beginning the proof, let us define a notion that will be used through the proof.  

\begin{definition}\label{def:packing}~\\
Let a set $\mathcal{F}\subset supp(P_{X})$.  Let $\lbrace x_1,\ldots,x_m\rbrace\subset \mathcal{F}$ and $p$ $(0,1]$. We say that the set $\lbrace x_1,\ldots,x_m\rbrace\subset \mathcal{F}$ is a $p$-probability-packing set of $\mathcal{F}$ if: 

\begin{equation}\label{eq:packing}
\forall s,s'\leq m,\;s\neq s'\Longrightarrow \rho(x_s,x_{s'})>r_p(x_s)\vee r_p(x_{s'})
\end{equation}
where $r_p$ is defined by \eqref{eq:rp}, and $a\vee b= \max(a,b)$ for $a,b$ $\in$ $\mathbb{R}$
\end{definition}  
This notion of $p$-probability-packing comes from the Definition 1.4 in \citep{edgar2000packing}. It will be used on a particular set of the form $\lbrace x\in\,supp(P_X),\;\gamma\leq\vert\eta(x)-\frac 12\vert\leq \gamma' \rbrace$ (where $0<\gamma<\gamma'$). This allows us to upper bound the number of informative points where we have a very high confidence for inferring their labels.
\begin{proof}~\\
Let us consider the last update of $I$, the set of indexes of informative points used in \texttt{KALLS}(Algorithm \ref{algo:KALLS}).\\  Set $\displaystyle s_{I}=\max I$, the index of the last informative point. We consider two cases:
\begin{enumerate}
\item \underline{\textbf{First case: $s_I=w$}}: we can easily see that \eqref{necessary-condition-budget} is satisfied, and we have trivially that the condition \eqref{eq:label-complexity} is sufficient to guarantee \eqref{necessary-condition-budget}.
\item \underline{\textbf{Second case: $s_I<w$}}: then the total number of label requests up to $s_I$ is:
\begin{equation}
\sum_{s\in I}\vert Q_s\vert
\end{equation}
where $\vert Q_s\vert$ is the number of label requests used in the subroutine \texttt{ConfidentLabel}(Algorithm\eqref{algo:ConfidentLabel}) with input $X_s$. Let $s$ $\in$ $I$.
For brevity, let us denote \texttt{ConfidentLabel}$(X_s,t)$:=\texttt{ConfidentLabel}$( X_s, k(\epsilon,\delta_s), t, \delta_s)$, $($where $\displaystyle t=n-\sum_{s_i\in I,s_i<s} \vert Q_{s_i}\vert$ the budget parameter $)$.  If $s\neq s_I$, the subroutine $\texttt{ConfidentLabel}(X_s,t)$ implicitly assumes that the process of label request do not takes into account the constraint related to the budget $n$ $($very large budget with respect to $k(\epsilon,\delta_s))$ such that \texttt{ConfidentLabel}$(X_s,t)$=\texttt{ConfidentLabel}$(X_s,t=\infty)$ . Then we have: 
\begin{equation}\label{label-complexity-eq1}
n>\sum_{\substack{s\in I\\s<s_I}}\vert Q_s\vert
\end{equation} 

On the other hand, we want to guarantee the condition \eqref{necessary-condition-budget}. For this, necessary for all $s$ $\in$ $I$, such that $s\leq T_{\epsilon,\delta}$, and $s< s_I$,  at the end of the subroutine \texttt{ConfidentLabel}$(X_s,t)$, the budget $n$ is not yet reached and then we can replace the relation \eqref{label-complexity-eq1} by 
\begin{equation}\label{label-complexity-eq2}
n>\sum_{\substack{s\in I\\s< s_I\\s\leq T_{\epsilon,\delta}}}\vert Q_s\vert
\end{equation} 
Then, necessarily, \eqref{necessary-condition-budget} holds when \eqref{label-complexity-eq2} holds. \\
Also, for $s$ $\in$ $I$, by theorem\ref{theorem:savings label}, if we assume that $\vert\eta(X_s)-\frac 12\vert\geq \frac 12\Delta $, we have that $\vert Q_s\vert\leq \tilde{k}(\epsilon,\delta_s)$, and the subroutine \texttt{ConfidentLabel}$(X_s,t)$,  $($with $t=n-\sum_{s_i\in I,s_i<s} \vert Q_{s_i}\vert)$ terminates when the cut-off condition \eqref{guarantee-confident} is satisfied. The right hand side of \eqref{label-complexity-eq2} is equal to: 

\begin{equation}\label{label-complexity-eq3}
\sum_{\substack{s\in I\\s< s_I\\s\leq T_{\epsilon,\delta}\\ \vert\eta(X_s)-\frac 12\vert\geq \frac 12 \Delta }}\vert Q_s\vert + \sum_{\substack{s\in I\\s< s_I\\s\leq T_{\epsilon,\delta}\\ \vert\eta(X_s)-\frac 12\vert\leq \frac 12\Delta}}\vert Q_s\vert
\end{equation}
Firstly, let us consider the first term in \eqref{label-complexity-eq3} and denote it by $T_1$. Let us denote by $B_s$ the event: 
$$B_{s}=\lbrace \vert\eta(X_s)-\frac 12\vert\geq \frac 12\Delta\rbrace.$$
We have
\begin{equation}\label{label-complexity-eq4}
\mathds{1}_{B_s}=\sum_{j=1}^{m_{\epsilon}}\mathds{1}_{B_{s,j}}
\end{equation}
where $$B_{s,j}=\lbrace 2^{j-1} \frac{1}{2}\Delta\leq \vert\eta(X_s)-\frac 12\vert\leq 2^{j}\frac 12\Delta\rbrace\quad\text{and}\;m_{\epsilon}=\left\lceil\log_2\left(\frac{1}{\frac 12 \Delta}\right)\right\rceil.$$

Then, 
\begin{align}
T_1 &\leq\sum_{\substack{s\in I\\s< s_I\\s\leq T_{\epsilon,\delta}\\ \vert\eta(X_s)-\frac 12\vert\geq \frac 12 \Delta }}\tilde{k}(\epsilon,\delta_s)\;\quad \text{by Theorem \ref{theorem:savings label}}\nonumber\\
%  &\leq \sum_{\substack{s\in I\\s\leq s_I\\s\leq T_{\epsilon,\delta}\\ \vert\eta(X_s)-\frac 12\vert\geq \frac 12 \Delta }}k(\epsilon,\delta_s)\;\quad (\text{As}\; \tilde{k}_{s}(\epsilon,\delta)\leq k_{s}(\epsilon,\delta))\\
  &=\sum_{\substack{s\in I\\s< s_I\\s\leq T_{\epsilon,\delta} }}\sum_{j=1}^{m_\epsilon}\tilde{k}(\epsilon,\delta_s)\mathds{1}_{B_{s,j}}\label{label-complexity-eq5}
\end{align}
On $B_{s,j}$, 
\begin{align}
\tilde{k}(\epsilon,\delta_s)&\leq \frac{c}{2^{2j}\Delta^2}\left[\log(\frac{32s^2}{\delta})+\log\log(\frac{32s^2}{\delta})+\log\log\left(\frac{512\sqrt{e}}{2^j\Delta}\right)\right]\nonumber\\
                              & \leq \frac{c}{2^{2j}\Delta^2}\left[2\log(\frac{32s^2}{\delta})+\log\log\left(\frac{512\sqrt{e}}{\Delta}\right)\right]\label{label-complexity-eq6}                              
\end{align}
Then \eqref{label-complexity-eq5} becomes: 
\begin{align}
T_1 &\leq \frac{c}{\Delta^2}\left[2\log(\frac{32T_{\epsilon,\delta}^2}{\delta})+\log\log\left(\frac{512\sqrt{e}}{\Delta}\right)\right]\sum_{j=1}^{m_\epsilon}2^{-2j}\sum_{\substack{s\in I\\s\leq s_I\\s\leq T_{\epsilon,\delta} }}\mathds{1}_{B_{s,j}}\label{label-complexity-eq7}
% &\leq \frac{c}{\Delta^2}\left[2\log(\frac{32T_{\epsilon,\delta}^2}{\delta})+\log\log\left(\frac{512\sqrt{e}}%{\Delta}\right)\right]\sum_{j=1}^{m_\epsilon}2^{-2j}\sum_{s\leq T_{\epsilon,\delta}}\mathds{1}_{B_{s,j}}
\end{align}

In \eqref{label-complexity-eq7}, the term $N_j=\displaystyle\sum_{\substack{s\in I\\s\leq s_I\\s\leq T_{\epsilon,\delta} }}\mathds{1}_{B_{s,j}}$ represents the numbers of informative points  that belong to the set 
\begin{equation}
I_j=\lbrace x, \;\gamma_{j-1}\leq \vert \eta(x)-\frac 12\vert \leq \gamma_j\rbrace
\end{equation}
(where $\gamma_j=2^j.\frac{\Delta}{2}$, $j=1\ldots, m_{\epsilon}$). We will prove that  
\begin{equation}
N_j\leq O\left((\gamma_j)^{\beta-\frac{d}{\alpha}}\right)
\label{eq:packing-upper}
\end{equation}
We proceed in two steps: 
\begin{itemize}
\item The set of informative points that belong to $I_j$ forms a $p_j$-probability-packing set (for $p_j$ well chosen) of $I_j$.
\item The cardinal of any $p_j$-probability-packing set satisfies \eqref{eq:packing-upper} 
\end{itemize}
\begin{enumerate}
\item Let us begin with first step: \\
Let $X_s,X_{s'}$ any two informative points that belong to $I_j$. Let us assume that $s<s'$. As $X_s$ $\in$ $I_j$, we have $\vert \eta(X_s)-\frac{1}{2}\vert\geq \frac{\Delta}{2}$ and by Theorem \ref{theorem:savings label}, the number of label requests $\bar{k_s}$ used in \texttt{ConfidentLabel}$(X_s)$ satisfies:

\begin{equation}
\vert \widehat{\eta}_{\bar{k}_s}-\frac{1}{2}\vert \geq 2b_{\delta_s,\bar k_s}
\label{eq:label-cutoff}
\end{equation}   
where $\widehat{\eta}_{\bar{k}_s}:=\widehat{\eta}_{\bar{k}_s}(X_s)$ and $b_{\delta_s,\bar k_s}$ are respectively defined by \eqref{eq:eta_Xs} and \eqref{eq:pi_ks}.\\
Then 
\begin{equation}
LB_s:=\vert \widehat{\eta}_{\bar{k}_s}-\frac{1}{2}\vert 
- b_{\delta_s,\bar k_s}\geq 0.1 b_{\delta_s,\bar k_s}
\label{eq:lowergarant} 
\end{equation}
Additionally, as $X_s$ and $X_{s'}$ are both informative points, by Lemma \ref{lemma:sufficient-informative}, we necessary have on event $A_3$ (see Lemma \ref{lemma:sufficient-informative}),  that 
\begin{equation}
\min(\widehat{p}_X,\widehat{p}_{X'})\geq \frac{75}{94}\left(\frac{1}{64L}\widehat{LB}_s\right)^{d/\alpha}
\label{eq:notberest}
\end{equation} 

On the event $A_3$,  the equations \eqref{eq:lowergarant}, \eqref{eq:notberest}, necessary imply: 
\begin{equation}
\min(P_X(B(X_{s},\rho(X_{s'},X_s))),P_X(B(X_{s'},\rho(X_{s'},X_s))))\geq \frac{28}{47}\left(\frac{1}{64L}\widehat{LB}_{s'}\right)^{d/\alpha}
\label{eq:lowerprob}
\end{equation}

Let us introduction the quantity $\bar\eta_{\bar{k}_s}:=\bar\eta_{\bar{k}_s}(X_s)$ defined by \eqref{eq:eta_Xs}. We have, by Theorem \ref{theorem:savings label}, on the event $A_2$ (see Theorem \ref{theorem:savings label}), 
$$\vert  \bar\eta_{\bar{k}_s}-\widehat{\eta}_{\bar{k}_s}\vert\leq b_{\delta_s,\bar k_s}.$$
Then, on the event $A_1\cap A_2$, we have:
\begin{align}
\vert\bar\eta_{\bar{k}_s}-\widehat{\eta}_{\bar{k}_s}\vert\leq b_{\delta_s,\bar k_s} \Rightarrow \vert\widehat{\eta}_{\bar{k}_s}-\frac 12\vert &\geq \vert  \bar\eta_{\bar{k}_s}-\frac 12\vert-b_{\delta_s,\bar k_s}\nonumber\\
                                             &\geq \vert\eta(X_s)- \frac 12\vert- \frac{31}{1024}\Delta-b_{\delta_s,\bar k_s}\nonumber\\
                                             &\;\;\;\; \text{by smoothness assumption (see for e.g \eqref{proof-sufficient k0})}\nonumber\\
                                             & \geq \vert\eta(X_s)- \frac 12\vert- \frac{62}{1024}\vert\eta(X_s)- \frac 12\vert-b_{\delta_s,\bar k_s}\nonumber\\
                                             &\;\;\;\quad \text{as\;}\vert\eta(X_s)- \frac 12\vert\geq \frac 12\Delta\\
                                             & =\frac{481}{512}\vert\eta(X_s)- \frac 12\vert-b_{\delta_s,\bar k_s}\label{eq:lowergarant1}
\end{align}

Therefore, we have on $A_1\cap A_2$: 

\begin{align}
LB_s &=\vert \widehat{\eta}_{\bar{k}_s}-\frac{1}{2}\vert 
- b_{\delta_s,\bar k_s}\nonumber\\
     & =\vert \widehat{\eta}_{\bar{k}_s}-\frac{1}{2}\vert 
- \frac 43b_{\delta_s,\bar k_s}+\frac 13b_{\delta_s,\bar k_s}\nonumber\\
     & \geq \frac 13\vert \widehat{\eta}_{\bar{k}_s}-\frac{1}{2}\vert + \frac 13 b_{\delta_s,\bar k_s}\nonumber\\
     & \text{\;\;by using \eqref{eq:label-cutoff}}\nonumber\\
     &\geq \frac 13\left(\frac{481}{512}\vert\eta(X_s)- \frac 12\vert-b_{\delta_s,\bar k_s}\right)+ \frac 13 b_{\delta_s,\bar k_s}\nonumber\\
     &\text{\;\;\;\;by using \eqref{eq:lowergarant1}}\nonumber\\
     &=\frac{481}{1536}\vert\eta(X_s)- \frac 12\vert\nonumber\\
     & \geq \frac{481}{1536} \gamma_{j-1}\;\text{as}\;X_s\in I_j \nonumber\\
     &=  \frac{481}{3072} \gamma_{j}\label{eq:lowergaranpac}
\end{align}

Then, the equation \eqref{eq:lowergaranpac} becomes:
\begin{equation}
\min(P_X(B(X_{s},\rho(X_{s'},X_s))),P_X(B(X_{s'},\rho(X_{s'},X_s))))\geq \frac{28}{47}\left(\frac{1}{L}\frac{481}{196608} \gamma_{j}\right)^{d/\alpha}
\label{eq:lowergaranlast}
\end{equation}
As the same way, we also obtain \eqref{eq:lowergaranlast} if $s'<s$.\\
Then, if we set 
\begin{equation}
p_j=\frac{28}{47}\left(\frac{1}{L}\frac{481}{196608} \gamma_{j}\right)^{d/\alpha}
\label{eq:probpacking}
\end{equation}
we have that, by\eqref{eq:rp} and \eqref{eq:lowergaranlast}:
\begin{equation}
\rho(X_{s'},X_s)\geq r_{p_j}(X_s)\vee r_{p_j}(X_{s'})
\label{eq:pack_j}
\end{equation} 
Therefore, the set of informative points that belong to $I_j=\lbrace x,\; \gamma_{j-1}\leq\vert\eta(x)-\frac 12\vert\leq \gamma_j\rbrace $ forms an $p_j$-probability-packing set.\\

\item As second step, let us determine an upper bound of the cardinal of any $p_j$-probability-packing set of $I_j$. Let $\Lambda_j=\lbrace x_1,\ldots, x_{M_j}\rbrace$ any $p_j$-probability-packing set of $I_j$.\\
For all $s,s'\leq M_j$, we obviously have: 
\begin{equation}
s\neq s'\Longrightarrow B(x_s,\frac{r_{p_j}(x_s)}{2})\cap B(x_{s'},\frac{r_{p_j}(x_{s'})}{2})=\emptyset
\end{equation} 
Then, we have:
\begin{align} 
P_X(\bigcup_{s=1}^{M_j}B(x_s,r_{p_j}(x_s)/2))&= \sum_{i=1}^{M_j} P_X(B(x_s,r_{p_j}(x_s)/2))\nonumber\\
                                                  &\geq C_{db}\sum_{i=1}^{M_j} P_X(B(x_s,r_{p_j}(x_s)))\nonumber\\
                                                  &\;\;\;\;\;\; \text{\;\;\;by assumption \eqref{def:doubling}}\nonumber\\
                                                  &\geq C_{db}M_j p_j \label{eq:upperMj}\\
                                                  &\text{by \eqref{eq:rp}}\nonumber 
\end{align}

On the other hand, if $z$ $\in$ $B(x_s,r_{p_j}(x_s)/2)$ for some $s\leq M_j$,  by the assumption \eqref{def:smooth}, equation \eqref{eq:rp}, and the fact that $\vert \eta(x_s)-\frac 12\vert\leq \gamma_j$, we have: 

\begin{align}
\vert\eta(z)-\frac 12\vert &\leq \gamma_j+L(p_j)^{\alpha/d}\nonumber\\
                            &= \gamma_j+ (\frac{28}{47})^{\alpha/d}\frac{481}{196608}\gamma_{j}\nonumber\\
                            &=\gamma_j\left(1+(\frac{28}{47})^{\alpha/d}\frac{481}{196608}\right)\nonumber\\
                            &=\tilde{c}\gamma_j\label{eq:uppertsy}
\end{align}
Where 
\begin{equation}
\tilde{c}=\left(1+(\frac{28}{47})^{\alpha/d}\frac{481}{196608}\right)
\end{equation}
Now we can upper bound $M_j$: by using \eqref{eq:uppertsy}, \eqref{eq:upperMj}, and assumption \eqref{def:TsybakovMarginNoise}, 
\begin{align}
C_{db}M_j p_j\leq P_X(\bigcup_{s=1}^{M_j}B(x_s,r_{p_j}(x_s)/2))&\leq P_X( z\; \vert\eta(z)-\frac 12\vert\leq \tilde{c}\gamma_j)\nonumber\\
                                                               & \leq C(\tilde{c}\gamma_j)^\beta
\end{align}

Then, 
\begin{align}
M_j &\leq \frac{C}{C_{db}}\frac{(\tilde{c}\gamma_j)^\beta}{p_j}\nonumber\\
   &= \tilde{b}(\gamma_j)^{\beta-\frac{d}{\alpha}}\label{eq:finaluppercov}
\end{align} 
where $\displaystyle\tilde{b}=  \frac{C}{C_{db}}\tilde{c}^{\beta-\frac{d}{\alpha}}\frac{47}{28}\left(\frac{196608}{481}\right)^{d/\alpha}.$\\
Then, the cardinal of any $p_j$-probability-packing set of $I_j$ is upper bound by $O\left((\gamma_j)^{\beta-\frac{d}{\alpha}}\right)$, consequently, equation \eqref{eq:packing-upper} holds.
\end{enumerate}
The equation \eqref{label-complexity-eq7} becomes:

\begin{align}
T_1 &\leq \frac{c\tilde{b}}{\Delta^2}\left[2\log(\frac{32T_{\epsilon,\delta}^2}{\delta})+\log\log\left(\frac{512\sqrt{e}}{\Delta}\right)\right]\sum_{j=1}^{m_\epsilon}2^{-2j}(\gamma_j)^{\beta-\frac{d}{\alpha}}\nonumber\\
    & =\frac{c\tilde{b}}{2}\Delta^{\beta-\frac{d}{\alpha}-2}\left[2\log(\frac{32T_{\epsilon,\delta}^2}{\delta})+\log\log\left(\frac{512\sqrt{e}}{\Delta}\right)\right]\sum_{j=1}^{m_\epsilon}2^{(-2+\beta-\frac{d}{\alpha})j}\nonumber\\
T_1 &= \frac{c\tilde{b}}{\Delta^2}\left[2\log(\frac{32T_{\epsilon,\delta}^2}{\delta})+\log\log\left(\frac{512\sqrt{e}}{\Delta}\right)\right]\sum_{j=1}^{m_\epsilon}2^{-2j}(\gamma_j)^{\beta-\frac{d}{\alpha}}\nonumber\\
    & =b_0\left(\frac{1}{\epsilon}\right)^{\frac{2\alpha+d-\alpha\beta}{\alpha(\beta+1)}}\left[2\log(\frac{32T_{\epsilon,\delta}^2}{\delta})+\log\log\left(\frac{512\sqrt{e}}{\Delta}\right)\right]m_\epsilon\label{eq:upperlabelfinal2}    
\end{align} 
where $b_0=c\tilde{b}(2C)^{\frac{2\alpha+d-\alpha\beta}{\alpha(\beta+1)}}.$ Equation \eqref{eq:upperlabelfinal2} holds because we have $\alpha\beta\leq d$, $\Delta=\max\left(\frac{\epsilon}{2}, \left(\frac{\epsilon}{2C}\right)^{\frac{1}{\beta+1}}\right)$.

Now, it remains to upper bound the second term in \eqref{label-complexity-eq3}. Denote it by $T_2$. 
By Lemma \ref{Chernoff}, (equation\eqref{eq:chernoff2}), there exists an event $A_5$ such that $P(A_5)\geq 1-\delta/8$, and on $A_5$, we have:
$$T_2\leq\sum_{s\leq  T_{\epsilon,\delta}}\vert Q_s\vert\mathds{1}_{\tilde{B}_{s}} \leq  k(\epsilon,\delta)\left(1+\frac{4}{P_X(\tilde{B})T_{\epsilon,\delta}}\log\left(\frac{8}{\delta}\right)\right)P_X(\tilde{B})T_{\epsilon,\delta}$$
Because $\vert Q_s\vert\leq k(\epsilon,\delta)$ (according to the subroutine \texttt{ConfidentLabel}) for all $s\leq T_{\epsilon,\delta}$ and where $\tilde{B}=\lbrace x,\;\vert \eta(x)-\frac 12\vert\leq \frac{\Delta}{2}\rbrace$ and $k(\epsilon,\delta)$ is defined in \eqref{eq:k_s}.\\
Consequently, we have:
\begin{align}
 T_2 &\leq k(\epsilon,\delta)\left(P_X(\tilde{B})T_{\epsilon,\delta}+4\log\left(\frac{8}{\delta}\right)\right)\nonumber\\
     &\leq k(\epsilon,\delta)\left(T_{\epsilon,\delta}\frac{1}{2^{\beta}}C\Delta^{\beta}+ 4\log\left(\frac{8}{\delta}\right)\right)\quad \text{by assumption \eqref{def:TsybakovMarginNoise}}\nonumber\\
     &=k(\epsilon,\delta)\left(\left(\frac{128L}{\Delta}\right)^{d/\alpha}\log\left(\frac{8}{\delta}\right)\frac{1}{2^{\beta}}C\Delta^{\beta}+4\log\left(\frac{8}{\delta}\right)\right)\quad \text{by \eqref{eq:label-instance}}\nonumber\\
     &=k(\epsilon,\delta)\log\left(\frac{8}{\delta}\right)\left(\left(128\right)^{d/\alpha-\beta}64^{\beta}C\left(\frac{1}{\Delta}\right)^{d/\alpha-\beta}+4\right)\label{label-complexity-eq8}
\end{align}
As $\alpha\beta\leq d$, $\Delta\leq1$, $C\geq 1$, the term $\left(128\right)^{d/\alpha-\beta}64^{\beta}C\left(\frac{1}{\Delta}\right)^{d/\alpha-\beta}$ in \eqref{label-complexity-eq8} is greater than $1$. Thus, \eqref{label-complexity-eq8} becomes:

\begin{align}
T_2 &\leq 5k(\epsilon,\delta)\log\left(\frac{8}{\delta}\right)\left(128\right)^{d/\alpha-\beta}64^{\beta}C\left(\frac{1}{\Delta}\right)^{d/\alpha-\beta}\nonumber\\
    &=5c\left[\log(\frac{1}{\delta})+\log\log(\frac{1}{\delta})+\log\log\left(\frac{512\sqrt{e}}{\Delta}\right)\right]\log\left(\frac{8}{\delta}\right)\left(128\right)^{d/\alpha-\beta}64^{\beta}C\left(\frac{1}{\Delta}\right)^{d/\alpha-\beta+2}\nonumber\\
    &\quad\;\;\;\;\;\; \text{see \eqref{eq:k_s}, where}\; k(\epsilon,\delta)\; \text{is defined}\nonumber\\
    &\leq \left(\frac{1}{\epsilon}\right)^{\frac{2\alpha+d-\alpha\beta}{\alpha(\beta+1)}}\left[\log(\frac{1}{\delta})+\log\log(\frac{1}{\delta})+\log\log\left(\frac{512\sqrt{e}}{\Delta}\right)\right]\log\left(\frac{8}{\delta}\right)\tilde{u}\label{label-complexity-eq9}
\end{align}
Where $\tilde{u}=5c\left(2C\right)^{\frac{2\alpha+d-\alpha\beta}{\alpha(\beta+1)}}64^{\beta}(128)^{d/\alpha-\beta}C$. The equation \eqref{label-complexity-eq9} holds  by using the definition of $\Delta$ \eqref{eq:margin1}.\\

By combining \eqref{label-complexity-eq9} and \eqref{eq:upperlabelfinal2} , the term obtained in \eqref{label-complexity-eq3} is less than: 
\begin{align*}
& b_0\left(\frac{1}{\epsilon}\right)^{\frac{2\alpha+d-\alpha\beta}{\alpha(\beta+1)}}\left[2\log(\frac{32T_{\epsilon,\delta}^2}{\delta})+\log\log\left(\frac{512\sqrt{e}}{\Delta}\right)\right]m_\epsilon+\\
& \left(\frac{1}{\epsilon}\right)^{\frac{2\alpha+d-\alpha\beta}{\alpha(\beta+1)}}\left[\log(\frac{1}{\delta})+\log\log(\frac{1}{\delta})+\log\log\left(\frac{512\sqrt{e}}{\Delta}\right)\right]\log\left(\frac{8}{\delta}\right)\tilde{u}
\end{align*}
Thus, if the label budget $n$ satisfies
\begin{align}
n\geq b_0\left(\frac{1}{\epsilon}\right)^{\frac{2\alpha+d-\alpha\beta}{\alpha(\beta+1)}}\left[\left[2\log(\frac{32T_{\epsilon,\delta}^2}{\delta})+\log\log\left(\frac{512\sqrt{e}}{\Delta}\right)\right]m_\epsilon+\log(\frac{1}{\delta})+\log\log(\frac{1}{\delta})+\log\log\left(\frac{512\sqrt{e}}{\Delta}\right)\right]\log\left(\frac{8}{\delta}\right)\tilde{u}\label{eq:finalupperlabelcompl}
\end{align}

We have that $n$ satisfies \eqref{label-complexity-eq2}, and \eqref{necessary-condition-budget} is necessary satisfied. 
\end{enumerate} 
\end{proof}

\textbf{Proof of Theorem~\ref{upperboundcomplexity}}~\\
Finally, we are able to prove Theorem \ref{upperboundcomplexity}. \\
Previously, we have designed five events $A_1,A_2,A_3,A_4,A_5$, such that if the event $A:=A_1\cap A_2\cap A_3\cap A_4\cap A_5$, and if the label budget $n\geq \tilde{O}\left(\left(\frac{1}{\epsilon}\right)^{\frac{2\alpha+d-\alpha\beta}{\alpha(\beta+1)}}\right)$ \eqref{eq:finalupperlabelcompl} (where $\tilde{O}$ assumes the existence of polylogarithmic factor in $\frac{1}{\epsilon}$ and $\frac{1}{\delta}$), then \eqref{necessary-condition-budget} is necessary satisfied, therefore , by Theorem \ref{theorem:lab-instance-space}, if  $w$ satisfies \eqref{guarantee-on-pool-0} and \eqref{condition2-0}, and the event $A$ holds, the classifier provided by the algorithm \texttt{KALLS} agrees with the Bayes classifier on $\lbrace x, \;\vert\eta(x)-\frac{1}{2}\vert> \Delta\rbrace$. Consequently, we have: 

\begin{align*}
R(\widehat{f}_{n,w})-R(f^*)&=\int_{supp(P_X)} \mathds{1}_{\widehat{f}_{n,w}\neq f^*(x)}\vert 2\eta(x)-1\vert dP_X(x)\\
                           &\leq \int_{\lbrace x, \;\vert\eta(x)-\frac{1}{2}\vert\leq \Delta\rbrace\cap supp(P_X)} 2\Delta \,dP_X(x)\\
                           &\leq 2C\Delta^{\beta+1}\quad \text{by assumption \eqref{def:TsybakovMarginNoise}}                        \\
                           &\leq \epsilon \quad\text{by using the definition of}\; \Delta \;\eqref{eq:margin1}
\end{align*}
Thus with probability at least $1-P(A^c)\geq 1-\delta$, \eqref{eq:error} holds.

\end{document}